\documentclass[final]{alt2026}

\usepackage[utf8]{inputenc} 
\usepackage[T1]{fontenc}    
\usepackage{hyperref}       
\usepackage{url}            
\usepackage{booktabs}       
\usepackage{amsfonts}       
\usepackage{nicefrac}       
\usepackage{microtype}      
\usepackage{xcolor}         
\usepackage{amsmath}
\usepackage{mathtools}
\usepackage[normalem]{ulem}
\usepackage{enumerate}
\usepackage{csquotes}
\usepackage{algorithm}
\usepackage{algpseudocode}
\usepackage{tabulary,makecell}

\def\argmin{\mathop{\mbox{ arg\,min}}}
\def\argmax{\mathop{\mbox{ arg\,max}}}

\newcommand{\deff}{d_{\mathrm{eff}}}

\definecolor{PalePurp}{rgb}{0.66,0.57,0.66}

\newcommand{\inner}[2]{\langle #1, #2 \rangle}

\title[Sparse Nonparametric Contextual Bandits]{Sparse Nonparametric Contextual Bandits}
\usepackage{times}

\altauthor{%
	\Name{Hamish Flynn} \Email{hamishflynn.gm@gmail.com}\\
	\addr Universitat Pompeu Fabra
	\AND
	\Name{Julia Olkhovskaya} \Email{julia.olkhovskaya@gmail.com}\\
	\addr Delft University of Technology%
	\AND
	\Name{Paul Rognon-Vael} \Email{paul.rognon@gmail.com}\\
	\addr Bocconi University%
}

\begin{document}

\maketitle

\begin{abstract}
We study the benefits of sparsity in nonparametric contextual bandit problems, in which the set of candidate features is countably or uncountably infinite. Our contribution is two-fold. First, using a novel reduction to sequences of multi-armed bandit problems, we provide lower bounds on the minimax regret, which show that polynomial dependence on the number of actions is generally unavoidable in this setting. Second, we show that a variant of the Feel-Good Thompson Sampling algorithm enjoys regret bounds that match our lower bounds up to logarithmic factors of the horizon, and have logarithmic dependence on the effective number of candidate features. When we apply our results to kernelised and neural contextual bandits, we find that sparsity enables better regret bounds whenever the horizon is large enough relative to the sparsity and the number of actions.
\end{abstract}

\section{Introduction}

The contextual bandit problem is a general model for sequential decision-making problems, in which at each step, a learner observes a context, plays an action in response to the context and then receives a reward. The goal of the learner is to maximise the reward accumulated over $n$ rounds, which is usually measured by the regret with respect to playing the best action for each context. The contextual bandit problem has attracted a great deal of attention because it is a faithful model of many real-world problems, such as personalised advertising \citep{abe1999learning}, personalised news recommendation \citep{li2010contextual} and medical treatment \citep{durand2018contextual}. In many practical situations, the set of possible contexts is very large, and the learner must use some sort of function approximation to learn general patterns that apply to new contexts. Previous works have considered finite-dimensional linear function approximation \citep{abe1999associative, dani2008stochastic, abbasi2011improved}, nonparametric function approximation \citep{bubeck2008online, kleinberg2008multi, srinivas2010gaussian, valko2013finite, slivkins2014contextual, singh2021continuum} and wide neural networks \citep{zhou2020neural, zhang2021neural, kassraie2022neural}. However, none of these approaches is entirely satisfactory. While linear methods lead to efficient estimation of the reward function, they typically only work well when one has considerable prior knowledge about the relationship between contexts, actions and rewards. In particular, the user is required to specify a (small) set of features such that the reward function is a linear combination of these features. Nonparametric methods are much more flexible, but they suffer from a curse of dimensionality. If the contexts are vectors, then the regret of a nonparametric contextual bandit algorithm can grow exponentially with the dimension of the contexts. Neural contextual bandit algorithms typically operate in the lazy regime \citep{jacot2018ntk, chizat2019lazy}, in which neural networks behave like kernel methods. As a result, these algorithms suffer from the same drawbacks as nonparametric methods. Can we achieve the best of both worlds: a contextual bandit algorithm that selects a small set of useful features from an infinite set of candidate features to achieve both flexibility and sample efficiency?

In this work, we study a class of contextual bandit problems that we call \emph{sparse nonparametric contextual bandits}. Briefly, we assume 
that the expected reward for each context-action pair is a linear combination of features selected from a set of infinitely many candidate features, which enables flexible function modelling. We further assume a sparse structure opening the way to efficient estimation. We consider both the case of \textit{countable sparsity}, where the set of candidate features is countable, and the case of \textit{uncountable sparsity}, where the set is uncountable. We fully describe the class under study in Section \ref{defSNPB}.
\paragraph{Contributions.} We consider sparse nonparametric contextual bandits with $n$ rounds, $K$ actions per round and sparsity $s$. In the uncountable sparsity model, we consider candidate features parameterised by a $d$-dimensional vector. Our contribution is as follows:
\begin{itemize}

\item First, using a novel reduction to sequences of multi-armed bandit problems, we establish lower bounds of order $\sqrt{Ksn}$ for countable sparsity and $\sqrt{Ksdn}$ for uncountable sparsity. These lower bounds show that it is not possible to achieve low worst-case regret when the number of actions is large relative to the number of rounds.
\item Second, we propose an algorithm based on the Feel-Good Thompson Sampling (FGTS) algorithm \citep{zhang2022feel}, which uses novel sparsity priors for nonparametric models. We prove that this algorithm enjoys regret bounds that nearly match our lower bounds, even when the sparsity $s$ is unknown. When the sparsity is known, our results demonstrate that FGTS is minimax optimal, up to logarithmic factors of $n$ (cf. Section 4).
\item In addition, we identify regimes in which sparsity enables improved regret bounds in kernelised and neural contextual bandit problems.
\end{itemize}
\paragraph{Outline.} The remainder of the paper is structured as follows. In Section \ref{sec:setting}, we formally describe the setting of sparse nonparametric contextual bandits and introduce some mild regularity conditions.
We also show
that our framework includes interesting settings that require feature learning, such as sparse kernelised contextual bandits and neural contextual bandits. In Section \ref{sec:lbs},  we state our lower bounds on the minimax regret and outline their proofs. In Section \ref{sec:ubs}, we describe FGTS with our sparsity priors, and we show that it enjoys regret bounds that closely match our lower bounds. Finally, in Section \ref{sec:discussion}, we summarise our findings. The proofs and the discussions of related work, limitations  and directions for future work are gathered in the Appendix. Additionally, in Appendix \ref{sec:kbs}, we use our upper and lower bounds on the minimax regret to identify regimes in which sparsity is helpful for regret minimisation in kernelised and neural contextual bandits.

\paragraph{Notation.}
For any $x \in \mathbb{R}$, $\lfloor x\rfloor$ is the greatest integer that is less than or equal to $x$ and $\lceil x\rceil$ is the least integer that is greater than or equal to $x$. For any positive integer $d$, $[d]$ is the set $\{1,\ldots,d\}$. For positive integers $a$ and $m$, $a ~\mathrm{mod}~ m$ is defined to be the unique integer $r \in \{0, 1, \dots, m-1\}$, such that $a = qm + r$, for some non-negative integer $q$. For any $q \in [1, \infty]$, $d \in \mathbb{N}$ and $R > 0$, we let $\mathbb{B}_q^d(R) = \{\theta \in \mathbb{R}^d: \|\theta\|_q \leq R\}$ denote the $d$-dimensional $\ell_q$-ball with radius $R$. For any set $\mathbb{T}$ and any $\epsilon > 0$, we use $\mathcal{M}(\mathbb{T}, \|\cdot\|, \epsilon)$ to denote the $\epsilon$-packing number of $\mathbb{T}$ (w.r.t.\ the norm $\|\cdot\|$) and $\mathcal{N}(\mathbb{T}, \|\cdot\|, \epsilon)$ to denote the $\epsilon$-covering number of $\mathbb{T}$.

\section{Problem Setting}
\label{sec:setting}

\subsection{Contextual Bandits}

We consider the following contextual bandit protocol, in which a learner interacts with an environment over a sequence of $n$ rounds. At the start of each round $t \in [n]$, the environment reveals a context $X_t \in \mathcal{X}$. In response, the learner selects an action $A_t \in \mathcal{A}$ and observes a real-valued reward $Y_t$. We let $\mathcal{F}_t = \sigma(X_1, A_1, Y_1, \dots, X_t, A_t, Y_t)$ denote the $\sigma$-field generated by the interaction history between the learner and the environment up to the end of round $t$, and we introduce the shorthand $\mathbb{E}_t[\cdot] = \mathbb{E}[\cdot|\mathcal{F}_{t-1}, X_t, A_t]$.

We assume that the action set is $\mathcal{A} = [K]$, and that each context $X_t$ is of the form $X_t = (X_{t,a})_{a \in \mathcal{A}}$, where each $X_{t,a}$ lies in some set $\mathcal{Z}$. For an arbitrary $x \in \mathcal{X}$ and $a \in [K]$, we use $x_a$ to denote the element in $\mathcal{Z}$ corresponding to the context $x$ and the action $a$. We allow the contexts to be selected by an adaptive adversary, which means that the environment can take the history $\mathcal{F}_{t-1}$ into account before selecting $X_t$. The rewards are assumed to be stochastic, and of the form $Y_t = f^{*}(X_{t},A_t) + \epsilon_t$, where $f^{*}: \mathcal{X} \times \mathcal{A} \to \mathbb{R}$ is a fixed reward function and $\epsilon_t$ is zero-mean, conditionally sub-Gaussian noise, meaning $\mathbb{E}_t[\exp(\lambda\epsilon_t)] \leq \exp(\lambda^2/8)$ for all $\lambda \in \mathbb{R}$.\footnote{More precisely, this means that $\epsilon_t$ is conditionally $\frac{1}{2}$-sub-Gaussian. Our regret analysis applies to any sub-Gaussian parameter. We assume that it is $\frac{1}{2}$ so that Theorem \ref{thm:expregretbound} is consistent with Theorem 1 in \citet{zhang2022feel}.} The goal of the learner is to minimise the expected cumulative regret, which is defined as
\begin{equation}\label{eq:freqregret}
R_n(f^{*}) = \mathbb{E}\left[\sum_{t=1}^{n}\max_{a \in \mathcal{A}}\{f^{*}(X_t, a)\} - f^{*}(X_t, A_t)\right]\,.
\end{equation}
The conditional distribution of the action $A_t$, conditioned on $\mathcal{F}_{t-1}$ and $X_t$ is denoted by $\pi_t(\cdot|\mathcal{F}_{t-1}, X_t)$. We call the sequence $(\pi_t)_{t=1}^{n}$ the policy of a contextual bandit algorithm.

\subsection{Sparse Nonparametric Contextual Bandits}\label{defSNPB}

The reward function $f^*$ is assumed to be an unknown linear combination of $s$ features that belong to a known set of infinitely many candidate features. The sparsity $s$ may or may not be known in advance. We consider two notions of sparsity that we refer to as \emph{countable sparsity} and \emph{uncountable sparsity}. To describe them, we start by recalling the standard, parametric, sparse linear contextual bandit problem, and then we show how our framework extends it.

In sparse linear contextual bandits, the reward function can be expressed as a weighted sum $f^{*}(x, a) = \sum_{i=1}^{p}w_i^*\phi_i(x_a)$ of finitely many features $\phi_1, \dots, \phi_p$. It is assumed that the weight vector $w^* \in \mathbb{R}^{p}$ contains only $s$ non-zero elements. A natural way to make this model more flexible is to express the reward function as the infinite weighted sum $f^*(x, a) = \sum_{i=1}^{\infty}w_i^*\phi_i(x_a)$, where $w^*$ is now a parameter sequence, as opposed to a parameter vector. We say that such a reward function is $s$-sparse when $\|w^*\|_0=s$, where $s$ is finite and ideally small. That is, there exists a finite subset $S \subset \mathbb{N}$ of size $s$, such that for all $i \notin S$, $w_i^*=0$. We refer to this as \emph{countable sparsity}. We assume that $\|w^*\|_1 \leq 1$ and $\|\phi_i\|_{\infty} \leq 1$ for all $i \in \mathbb{N}$, which implies that $\|f^*\|_{\infty} \leq 1$.

One can view the features $\phi_1, \phi_2, \dots$ as a sequence of functions, each mapping $\mathcal{Z}$ to $\mathbb{R}$, or as a single function $\phi: \mathcal{Z} \times \mathbb{N} \to \mathbb{R}$ that maps any $x_a \in \mathcal{Z}$ and $i \in \mathbb{N}$ to the value $\phi_i(x_a)$. Adopting the latter view, we can further generalise this model by replacing the index $i \in \mathbb{N}$ with a continuous parameter $\theta \in \Theta \subset \mathbb{R}^d$. The reward function can now be written as $f^{*}(x, a) = \int_{\Theta}\phi(x_a, \theta)\mathrm{d}w^*(\theta)$, where $w^*$ is a signed measure on $\Theta$. We say that such a reward function is $s$-sparse if $w^*$ is a discrete measure that can be written as a sum of $s$ Dirac measures, where the $i$\textsuperscript{th} Dirac measure is weighted by $w_i^*$ and centred at $\theta_i^*$. That is, there exist $s \in \mathbb{N}$, $\theta_1^*, \dots, \theta_s^* \in \Theta$ and $w^* \in \mathbb{R}^{s}$, such that $f^{*}(x, a) = \sum_{i=1}^{s}w_i^*\phi(x_a, \theta_i^*)$. Since the set of candidate features is uncountable, we refer to this as \emph{uncountable sparsity}. The set of all functions of the form $f(x, a) = \sum_{i=1}^{\infty}w_i\phi(x_a, \theta_i)$ contains every $s$-sparse reward function, so we will restrict our attention to functions of this form, where $w = (w_1, w_2, \dots)$ is a parameter sequence, and $\theta_1, \theta_2, \dots$ is a sequence of elements in $\Theta$. We assume that $\|w^*\|_1 \leq 1$, $\|\theta_i^*\|_2 \leq 1$ for all $i \in [s]$ and $\|\phi\|_{\infty} \leq 1$, which implies that $\|f^*\|_{\infty} \leq 1$.

We introduce the following shared notation to describe the classes of reward functions that we consider. We use $s$ and $S$ to denote the sparsity and the support of $f^*$. For countable sparsity, the support of $f^*$ is $S = \{i \in \mathbb{N}: w_i^* \neq 0\} \subseteq \mathbb{N}$. For uncountable sparsity, the support of $f^*$ is $S = \{\theta_1^*, \dots, \theta_s^*\}$. We use $m$ and $M$ to denote the sparsity and the support of an arbitrary model. We use $\nu$ to denote the parameter(s) of interest in both types of sparsity. That is, for countable sparsity $\nu = w$, and for uncountable sparsity $\nu = (w, \theta_1, \dots, \theta_m)$. Depending on the type of sparsity, we define $f_{\nu}$ to be the function $f_{\nu}(x, a) = \sum_{i=1}^{\infty}w_i\phi_i(x_a)$ or $f_{\nu}(x, a) = \sum_{i=1}^{m}w_i\phi(x_a, \theta_i)$. For countable sparsity, we use the notations $f_{\nu}$ and $f_{w}$ interchangeably.

\subsection{Regularity Conditions}
\label{sec:regularity}
In our analysis, we make the following regularity assumptions. For countable sparsity, we assume that the feature map satisfies one of following uniform decay conditions.

\begin{definition}[Uniform decay.]
We say that $(\phi_i)_{i=1}^{\infty}$ satisfies the polynomial decay condition if, for some $\beta > 1$, $\|\phi_i\|_{\infty} \leq i^{-\beta/2}$. We say that $(\phi_i)_{i=1}^{\infty}$ satisfies the exponential decay condition if, for some $\beta > 0$, $\|\phi_i\|_{\infty} \leq \exp(-i^{\beta}/2)$.\footnote{The factors of $\frac{1}{2}$ are to make these definitions consistent with the usual eigenvalue decay conditions for Mercer kernels (see for instance \citet{vakili2021information}). For positive constants $C$, $C_1$ and $C_2$, one can easily replace these conditions with $\|\phi_i\|_{\infty} \leq Ci^{-\beta/2}$ and $\|\phi_i\|_{\infty} \leq C_1\exp(-C_2i^{\beta}/2)$.}
\label{def:uni_decay}
\end{definition}
As discussed in Section \ref{sec:examples}, this is a natural assumption for kernelised bandits with countable sparsity. More generally, smoothness assumptions of this form appear frequently in nonparametric statistics \citep{wasserman2006all}. When the features satisfy one of these decay conditions, we can define a notion of \emph{effective dimension}. In particular, we define the effective dimension for sample size $n$ as
\begin{equation}\label{eq:effdim}
d_{\mathrm{eff}} := \min\{i \in \mathbb{N}: \forall j > i, \|\phi_j\|_{\infty}^2 \leq \tfrac{1}{n}\}.
\end{equation}
This definition ensures that if we approximate $f_{\nu}(x, a) = \sum_{i=1}^{\infty}w_i\phi_i(x_a)$ by the function $\tilde{f}_{\nu}(x, a) = \sum_{i=1}^{d_{\mathrm{eff}}}w_i\phi_i(x_a)$, then the squared approximation error is bounded by $1/n$.
In kernelised bandits, which will be one of our main points of reference (cf.\ Section \ref{sec:kbs}), $d_{\mathrm{eff}}$ is usually defined to be the effective degrees of freedom of the kernel ridge estimate, which is a data-dependent quantity (see e.g.\ \citet{valko2013finite, calandriello2019gaussian}). However, both definitions of $\deff$ are equivalent in the sense that, in the worst case, the scaling in $n$ of the effective degrees of freedom of the ridge estimate with the features $\phi_1, \phi_2, \dots$ matches that of our $\deff$ quantity in \eqref{eq:effdim}. One can verify that for polynomial decay, $\deff \leq n^{1/\beta}$, and for exponential decay, $\deff \leq \log^{1/\beta}(n)$. For uncountable sparsity, we require that the feature map is uniformly Lipschitz continuous. We assume that the Lipschitz constant is 1 for the sake of simplicity.  
\begin{definition}[Uniform Lipschitz continuity]
We say that $\phi$ is uniformly Lipschitz if
\begin{equation*}
\forall x \in \mathcal{X}, a \in [K], \theta, \theta^{\prime} \in \Theta, ~|\phi(x_a, \theta) - \phi(x_a, \theta^{\prime})| \leq \|\theta - \theta^{\prime}\|_2.
\end{equation*}\label{def:uni_lip} \end{definition} \vspace{-2em}  
As discussed in Section \ref{sec:examples}, this is a natural assumption for kernelised bandits with uncountable sparsity and neural bandits. Moreover, similar assumptions have been used in previous works on contextual bandits \citep{zhang2022feel, neu2022lifting, neu2024optimisic}.

\subsection{Examples}
\label{sec:examples}

We conclude this section by highlighting some specific instantiations of our framework.

\paragraph{Kernelised Bandits with Countable Sparsity.} In kernelised contextual bandits, each context-action pair corresponds to a vector $x_a \in \mathcal{Z} \subset \mathbb{R}^p$, and the reward function is $f^*(x, a) = h^*(x_a)$, where $h^*: \mathcal{Z} \to \mathbb{R}$ is a function in a reproducing kernel Hilbert space (RKHS) $\mathcal{H}$, with reproducing kernel $k: \mathcal{Z} \times \mathcal{Z} \to \mathbb{R}$. If $k$ is continuous and $\mathcal{Z}$ is a compact metric space, then $k$ is called a Mercer kernel, and due to Mercer's theorem, the kernel function can be written as
\begin{equation*}
k(z, z^{\prime}) = \sum_{i=1}^{\infty}\xi_i\varphi_i(z)\varphi_i(z^{\prime}),
\end{equation*}
where $(\xi_i)_{i=1}^{\infty}$ and $(\varphi_i)_{i=1}^{\infty}$ are the (non-negative) eigenvalues and eigenfunctions of the kernel (cf.\ Section 12.3 in \citet{wainwright2019high}). Moreover, $\mathcal{H}$ can be represented as
\begin{equation*}
\mathcal{H} = \{h(z) = {\textstyle \sum_{i=1}^{\infty}}w_i\sqrt{\xi_i}\varphi_i(z): {\textstyle \sum_{i=1}^{\infty}}w_i^2 < \infty\},
\end{equation*}
and the squared RKHS norm is $\|h\|_{\mathcal{H}}^2 = \sum_{i=1}^{\infty}w_i^2$. We notice that if we define $\phi_i := \sqrt{\xi_i}\varphi_i$, and consider functions in $\mathcal{H}$ corresponding to $s$-sparse sequences $w$, then we find ourselves in a sparse nonparametric contextual bandit problem, with countable sparsity. Note that for the commonly-used Mat\'{e}rn and RBF kernels, the eigenfunctions can be uniformly bounded and the eigenvalues decay to 0 as $i$ tends to $\infty$, which means the features $\phi_i$ will typically satisfy one of the uniform decay conditions in Definition \ref{def:uni_decay}.

\paragraph{Kernelised Bandits with Uncountable Sparsity.} The kernelised contextual bandit problem can also be modelled as a sparse nonparametric contextual bandit problem with uncountable sparsity. One can alternatively express the RKHS $\mathcal{H}$ as
\begin{equation*}
\mathcal{H} = \overline{\{h(z) = {\textstyle \sum_{i=1}^{\infty}}w_ik(z, z_i): {\textstyle \sum_{i=1}^{\infty}\sum_{j=1}^{\infty}}w_iw_jk(z_i, z_j) < \infty\}},
\end{equation*}
where for any set $A$, $\overline{A}$ denotes the closure of $A$. In addition, the squared RKHS norm can be expressed as $\|h\|_{\mathcal{H}}^2 = \sum_{i=1}^{\infty}\sum_{j=1}^{\infty}w_iw_jk(z_i, z_j)$. If we set $\Theta = \mathcal{Z}$, $\phi(z, \theta) = k(z, z^{\prime})$ and assume that the function $h^*$ is $s$-sparse, meaning $h^*(z) = \sum_{i=1}^{s}w_i^*k(z, z_i^*)$, then we find ourselves in a sparse nonparametric contextual bandit problem with uncountable sparsity (assuming $\mathcal{Z}$ is uncountable). The Mat\'{e}rn and RBF kernel functions are both Lipschitz, which means that the uniform Lipschitz continuity property is satisfied for these kernels.

\paragraph{Neural Bandits.} For our final example, we consider a class of contextual bandit problems in which the reward function is a single-layer neural network, of possibly unknown width. Each context-action pair corresponds to a vector $x_a \in \mathcal{Z} \subset \mathbb{R}^p$. The reward function can be written as the infinite-width neural network
\begin{equation*}
f^*(x, a) = \sum_{i=1}^{\infty}w_i^*\sigma(\inner{\theta_i^*}{x_a})\,,
\end{equation*}
where $\sigma: \mathbb{R} \to \mathbb{R}$ is an activation function and $\theta_i^* \in \Theta$. If we assume that $f^*$ is $s$-sparse, and define $\phi(z, \theta) = \sigma(\inner{\theta}{z})$, then this is a sparse nonparametric contextual bandit problem with uncountable sparsity (assuming $\Theta$ is uncountable). If the norm of $x_a$ is bounded for every $x \in \mathcal{X}$ and $a \in [K]$, and $\sigma$ is a Lipschitz activation function, then the uniform Lipschitz continuity property is satisfied.

\section{Minimax Lower Bounds}
\label{sec:lbs}

In this section, we establish lower bounds on the minimax regret for each type of sparsity. Both of our lower bounds are inspired by the lower bound for sparse linear bandits in Theorem 24.3 of \citet{lattimore2020bandit}, which exploits the fact that a sparse linear bandit problem can mimic a multi-task bandit problem. The main idea behind our lower bounds is that a sparse nonparametric contextual bandit problem, with either type of sparsity, can mimic a sequence of $K$-armed (non-contextual) bandit problems. In each of these sub-problems, there is a single good action with expected reward $\Delta$ and $K-1$ bad actions with expected reward $0$. The regret suffered in each sub-problem can be quantified using the lower bound for $K$-armed bandits in Exercise 24.1 of \citet{lattimore2020bandit} (see also Lemma \ref{lem:K_armed_lb}). Due to space limitations, we only describe the reduction to sequences of $K$-armed bandits in detail in Appendix \ref{app:lower_bound_proof}.

The lower bound obtained through this reasoning is typically larger when the original problem is split into a larger number of $K$-armed bandits, each with a smaller horizon. At one extreme, if we impose no restrictions on the features and the reward function other than boundedness, then the sparse nonparametric contextual bandit problem can be reduced a sequence of $n$ $K$-armed bandits, each with horizon 1, and with $\Delta = 1$. In this case, any algorithm is reduced to guessing in each round, and so the minimax regret is linear in $n$. If the features satisfy one of the regularity conditions in Section \ref{sec:regularity}, then this constrains the largest possible value of $\Delta$. In particular, the maximum value of $\Delta$ decreases as the number of $K$-armed bandits increases. One can still reduce the problem to a sequence of $n$ $K$-armed bandit problems, but $\Delta$ would be so small that the regret of any algorithm is negligible. The worst-case reduction is to a sequence of fewer than $n$ $K$-armed bandits, where some learning is possible, and so the minimax regret is sublinear. In the following subsections, we state our lower bounds for each type of sparsity.

\subsection{Countable Sparsity}
\label{sec:countable_lb}

In the countable sparsity setting, we obtain the following lower bound.

\begin{theorem}
Consider the sparse nonparametric contextual bandit problem with countable sparsity described in Section \ref{sec:setting}. Let $\mathcal{A} = [K]$ for some $K \geq 2$ and assume that the noise variables are standard Gaussian, i.e.\ $\epsilon_t \sim \mathcal{N}(0, 1)$. Suppose that for some $\beta > 1$ and some integer $m \geq s^{\beta + 2}K^{\beta + 1}$, $n = sm$. Then for any policy, there exists a sequence of contexts $x_1, \dots, x_n$, a parameter sequence $w = (w_1, w_2, \dots)$ with $\|w\|_0 = s$, $\|w\|_1 = 1$ and a sequence of functions $(\phi_i)_{i=1}^{\infty}$ with $\|\phi_i\|_{\infty} \leq i^{-\beta/2}$, such that
\begin{equation*}
R_n(f_{w}) \geq \frac{1}{8}\sqrt{Ksn}.
\end{equation*}

Instead, suppose that for some $\beta > 0$ and some integer $m \geq \lceil 1/\beta\rceil s^2K\exp(s^{\beta}K^{\beta\lceil 1/\beta\rceil})$, $n = sm$. Then for any policy, there exists a sequence of contexts $x_1, \dots, x_n$, a parameter sequence $w = (w_1, w_2, \dots)$ with $\|w\|_0 = s$, $\|w\|_1 = 1$ and a sequence of functions $(\phi_i)_{i=1}^{\infty}$ with $\|\phi_i\|_{\infty} \leq \exp(-i^{\beta}/2)$, such that
\begin{equation*}
R_n(f_{w}) \geq \frac{1}{8}\sqrt{\max(1, 1/\beta)Ksn}.
\end{equation*}
\label{thm:countable_lb}
\end{theorem}
Note that $R_n(f_{w})$ denotes the expected regret when the reward function is $f_{w}$, where $w$ is the difficult parameter sequence. The proof of this result can be found in Appendix \ref{sec:count_lb_proof}. Here, we provide some intuition about the manner in which the decay conditions constrain the number of $K$-armed bandits in the reduction. To reduce the problem to a sequence of $ms$ $K$-armed bandits, for any $m \geq 1$, we require that the features satisfy $\|\phi_i\|_{\infty} = \Delta$ for all $i \in [sK^m]$. Combined with either the polynomial or exponential decay condition, this means that $\Delta$ must satisfy either $\Delta \leq (sK^m)^{-\beta/2}$ or $\Delta \leq \exp((sK^m)^{-\beta}/2)$. Thus, there is a trade-off between choosing a large number of switches and a large gap $\Delta$. On the one hand, this lower bound does not rule out the possibility that there exist algorithms with regret bounds that depend polynomially on $s\log(d_{\mathrm{eff}})$. In the absence of sparsity, one would expect that the minimax regret depends polynomially on $\deff$. On the other hand, this lower bound has polynomial dependence on $K$, whereas in non-sparse linear or kernelised contextual bandits, the minimax regret depends only logarithmically on $K$. This suggests that sparsity is only helpful when the number of actions is sufficiently small (cf.\ Section \ref{sec:kbs}).

\subsection{Uncountable Sparsity}

In the uncountable sparsity setting, we obtain the following lower bound.

\begin{theorem}
Consider the sparse nonparametric contextual bandit problem with uncountable sparsity described in Section \ref{sec:setting}. Let $\mathcal{A} = [K]$ for some $K \geq 2$ and let $n = sdm$ for some integer $m \geq s^{2 + 2/d}K^{3}$. Assume that the noise variables are standard Gaussian. For any policy, there exists a sequence of contexts $x_1, \dots, x_n \in \mathcal{X}$, parameters $w \in \mathbb{B}_1^s(1)$, $\theta_1, \dots, \theta_s \in \Theta \subset \mathbb{B}_2^d(1)$ and a uniformly Lipschitz continuous function $\phi$, with $\|\phi\|_{\infty} \leq 1$, such that
\begin{equation*}
R_n(f_{\nu}) \geq \frac{1}{8}\sqrt{Ksdn}.
\end{equation*}\label{thm:uncountable_lb}
\end{theorem}
This lower bound has similar implications to the previous one. In particular, it does not rule out the possibility of an algorithm with a regret bound that is polynomial in $sd$, but it does suggest that sparsity is only helpful when the number of actions is sufficiently small. For the special case of $s = 1$, Theorem \ref{thm:uncountable_lb} provides a lower bound of order $\sqrt{Kdn}$ for the setting studied in \citet{neu2022lifting, neu2024optimisic}, where the reward function $f^*(x, a) = \phi(x_a, \theta^*)$ is a Lipschitz function of a $d$-dimensional parameter vector. The proof of Theorem \ref{thm:uncountable_lb} can be found in Appendix \ref{sec:uncount_lb_proof}. It uses a similar reduction to a sequence of $K$-armed bandits, though the number of switches is constrained in a different manner. For any given number of switches, the Lipschitz property is used to relate the maximum value of $\Delta$ to the packing number of the ball $\mathbb{B}_2^d(1)$.

\section{Matching Upper Bounds}
\label{sec:ubs}

In this section, we present an algorithm with regret bounds that match
the lower bounds in the previous section up to logarithmic factors of $n$. 
We first describe the algorithm, and then discuss the regret bounds for countable and uncountable sparsity separately.

\subsection{Method}

Our algorithm is a particular instance of the Feel-Good Thompson Sampling (FGTS) algorithm proposed by \citet{zhang2022feel}.  FGTS is a modification of the popular Thompson Sampling algorithm \citep{thompsonsampling} that achieves improved bounds on (frequentist) regret by more aggressively exploring promising actions. Our algorithm uses novel priors inspired by the PAC-Bayes literature on sparse regression to extend FGTS to sparse nonparametric contextual bandits.

To describe FGTS in detail, we must first introduce some notation. Let us define $a(\nu, x) = \argmax_{a\in[K]} f_{\nu}(x, a)$ to be the optimal action for context $x$ and reward function $f_{\nu}$. We let $f_{\nu}(x) = \max_{a \in [K]}\{f_{\nu}(x, a)\} = f_{\nu}(x, a(\nu, x))$ denote the maximum of $f_{\nu}(x, a)$ with respect to the action. Similarly, we let $a^*(x) = \argmax_{a\in[K]} f^*(x, a)$ and $f^*(x) = f^*(x, a^*(x))$. We use $p_1(\nu)$ to denote the prior distribution, which will be specified in the subsequent subsections. We consider the following negative log-likelihood for the reward, conditioned on $\nu$, $x$ and $a$:
\begin{align}\label{eq:fgtslikelihood}
L(\nu, x, a, y) = \eta(f_{\nu}(x, a) - y)^2 - \lambda f_{\nu}(x)\,.
\end{align}
The quadratic term $\eta(f_{\nu}(x, a) - y)^2$ corresponds to a Gaussian log-likelihood with mean $f_{\nu}(x, a)$ and variance $1/(2\eta)$. The Feel-Good exploration term, $\lambda f_{\nu}(x)$, favours parameters $\nu$ where the maximum of $f_{\nu}(x, a)$ w.r.t.\ the action is large. Finally, $\eta>0$ and $\lambda>0$ are tuning parameters, which we will set later. With the negative log-likelihood $L$ and the prior $p_1$, the posterior $p_t$ after $t-1$ rounds have been completed is
\begin{equation}\label{eq:postnu}
p_t(\nu) \propto \exp\Big(-\sum_{l=1}^{t-1} L(\nu, X_l, A_l, Y_l)\Big) p_1(\nu).
\end{equation}
In each round $t$, FGTS draws a sample $\nu_t$ from the posterior $p_t$, and then plays the action $A_t = a(\nu_t, X_t)$. We give the pseudocode for FGTS in Algorithm \ref{alg:FGTS}.

\begin{algorithm}[H]
\caption{Feel-Good Thompson Sampling}\label{alg:FGTS}
\begin{algorithmic}
\State \textbf{Input:} prior $p_1$, parameters $\eta$, $\lambda$
\For{$t=1,\dots,n$}
\State Observe context $X_t$,
\State Draw $\nu_t \sim p_t$ according to \eqref{eq:postnu},
\State Select action $A_t = a(\nu_t,X_t)$,
\State Observe reward $Y_t$.
\EndFor
\end{algorithmic}
\end{algorithm}
To bound the regret of FGTS, we use the decoupling technique developed by \citet{zhang2022feel}. As long as the prior satisfies $\mathbb{P}_{\nu \sim p_1}[\max_{x, a}|f_{\nu}(x, a)| \leq 1] = 1$, Theorem 1 of \citet{zhang2022feel} (see also Theorem \ref{thm:expregretbound}) states that
\begin{equation*}
R_n(f^{*})\leq \frac{\lambda K n}{\eta}+6 \lambda n-\frac{1}{\lambda} Z_n\,,
\end{equation*}
where $Z_n := \mathbb{E} \log \mathbb{E}_{\nu \sim p_1} \exp (-\sum_{t=1}^n\Delta L(\nu,X_t, A_t, Y_t))$ is a log-partition function with potential function $\Delta L(\nu, x, a, y) := \eta[(f_{\nu}(x, a)-y)^2-\left(f^{*}(x, a)-y\right)^2]-\lambda[f_\nu(x)-f^{*}(x)]$, which is the logarithm of a likelihood ratio. The main technical contribution of our regret analysis is to derive new bounds on $Z_n$, which depend on the true (possibly unknown) sparsity $\|w^*\|_0$. These new bounds on $Z_n$ may be of independent interest, since they can be used to prove regret bounds for sparse nonparametric regression (cf.\ chapters 2 and 3 in \citet{cesa2006prediction}). We state and prove our bounds on $Z_n$ in Appendix \ref{sec:cs_ub_Zn} and Appendix \ref{sec:us_ub_Zn}. In the following subsections, we describe the priors that we use and the regret bounds that we obtain for each type of sparsity.

\subsection{Countable Sparsity}
\label{sec:ub_cs}

For countable sparsity, we use a subset selection prior that can be factorised into a discrete distribution $p_1(M)$ over subsets $M \subseteq \mathbb{N}$, and a conditional distribution $p_1(w|M)$ over parameter sequences $w \in \mathbb{R}^{\infty}$ with support $M$. In particular, we adapt the prior used in Section 3 of \citet{alquier2011pac}, which is designed for finite-dimensional sparse linear models with a $p$-dimensional parameter vector. \citet{alquier2011pac} use a prior $p_1(M)$ over subsets $M \subseteq [p]$ that can be factorised into a distribution $p_1(m)$ over subset sizes $m \in [p]$ and a conditional distribution $p_1(M\mid m)$ over subsets $M \subseteq [p]$ of size $m$. To penalise large subsets and express no preference between any two subsets of the same size, $p_1(m)$ is chosen to be $p_1(m) \propto 2^{-m}$ (with $p_1(0) = 0$) and $p_1(M\mid m)$ is chosen to be the uniform distribution over subsets of size $m$. The resulting distribution over subsets of $[p]$ is
\begin{equation*}
p_1(M) = \sum_{m=1}^{p}p_1(m)p_1(M\mid m) = p_1(|M|)p_1(M \mid |M|) =  \frac{2^{-|M|}}{\binom{p}{|M|}\sum_{m=1}^{p}2^{-m}}\,,
\end{equation*}
for any non-empty $M \subseteq [p]$. The conditional distribution $p_1(w\mid M)$ over parameter vectors with support $M$ is the uniform distribution on the set $\mathcal{W}_M := \{w: \|w\|_1 \leq 1, ~\text{and}~ w_i = 0 ~\forall i \notin M\}$. The resulting prior over parameter vectors $w$ is $p_1(w) = \sum_{M \subseteq [p]}p_1(M)p_1(w\mid M)\,.$

A na\"{i}ve way of extending this prior to the infinite-dimensional case would be to replace the sums over $m$ from $1$ to $p$ with sums over $m$ from $1$ to infinity. However, this fails for two reasons. First, the power set of $\mathbb{N}$ is uncountable, so it is not clear that a sum over all subsets of $\mathbb{N}$ is well-defined. Second, for a fixed subset size $m$, the set of all subsets of $\mathbb{N}$ of size $m$ is countably infinite, so a uniform distribution on this set cannot be defined. Fortunately, these problems can be fixed by exploiting the uniform decay condition in Definition \ref{def:uni_decay}. In particular, we restrict the support of the distribution $p_1(M)$ to subsets $M \subseteq [\deff]$. Intuitively, if we include at least $\deff$ features, then the uniform decay condition ensures that the regret suffered by ignoring some features will be negligible compared to the regret suffered while estimating the best $\deff$-dimensional approximation of $w^*$. This intuition is made rigorous in the proof of Theorem \ref{thm:upperbounds_cs}. We set
\begin{equation*}
p_1(M) =  \mathbb{I}\{M \subseteq [\deff]\}\frac{2^{-|M|}}{\binom{\deff}{|M|}\sum_{m=1}^{\deff}2^{-m}}\,,
\end{equation*}
for non-empty $M \subseteq [\deff]$ and $p_1(\emptyset) = 0$, which penalises large subsets and assigns probability zero to any subset not contained within $[\deff]$. For any subset $M \subseteq [\deff]$, we choose the conditional distribution over parameter sequences with support $M$ to be the uniform distribution $\mathcal{U}(\mathcal{W}_{M})$ over $\mathcal{W}_M$, which has the density function
\begin{equation*}
p_1(w\mid M) = \tfrac{|M|!}{2^{-|M|}}\mathbb{I}\{\|w\|_1 \leq 1\}\mathbb{I}\{w_i = 0 ~~\forall i \notin M\}\,.
\end{equation*}
This reflects our assumption that $\|w^*\|_1 \leq 1$. The resulting prior over parameter sequences is
\begin{equation}
p_1(w) = \sum_{M \subseteq [\deff]}p_1(M)p_1(w\mid M)\,.\label{eqn:cs_prior}
\end{equation}
For every $w$ in the support of this prior, we have $\|w\|_1 \leq 1$. This means that, as required by Theorem \ref{thm:expregretbound}, we have $\mathbb{P}_{w \sim p_1}[\max_{x, a}|f_{w}(x, a)| \leq 1] = 1$. Our regret bound for FGTS in the countable sparsity setting is stated in the following theorem.

\begin{theorem}\label{thm:upperbounds_cs}
Consider FGTS with $\eta = 1/4$ and the prior $p_1$ defined in \eqref{eqn:cs_prior}. The expected regret of FGTS with $\lambda \propto \sqrt{\log(\deff n)/(Kn)}$ is at most
\begin{equation*}
R_n(f^*) = \mathcal{O}(\|w^*\|_0\sqrt{Kn\log(\deff n)})\,.
\end{equation*}
Suppose that $s \geq \|w^*\|_0$ is a known upper bound on the sparsity. The expected regret of FGTS with $\lambda \propto \sqrt{s\log(\deff n)/(Kn)}$ is at most
\begin{equation}
R_n(f^*) = \mathcal{O}(\sqrt{Ksn\log(\deff n)})\,.\label{eqn:cs_ub_s}
\end{equation}
\end{theorem}
This regret bound shows that when either of the uniform decay conditions in Definition \ref{def:uni_decay} is satisfied, FGTS with the prior $p_1$ is nearly minimax optimal. When the polynomial decay condition is satisfied, the effective dimension satisfies $\deff \leq n^{1/\beta}$. Therefore, the regret bound for FGTS is of order at most $\sqrt{(1 + 1/\beta)Ksn\log(n)}$. Since $\beta > 1$, $1 + 1/\beta$ less than 2, so this matches the lower bound in Theorem \ref{thm:countable_lb} up to a factor of $\sqrt{\log(n)}$. When the exponential decay condition is satisfied, $\deff \leq \log^{1/\beta}(n)$, and so the regret bound for FGTS is of order at most $\sqrt{\max(1,1/\beta)Ksn\log(n)}$. Once again, this matches the lower bound in Theorem \ref{thm:countable_lb} up to a factor of $\sqrt{\log(n)}$.

\subsection{Uncountable Sparsity}
\label{sec:ub_us}

For uncountable sparsity, we use a prior that can be factorised into a discrete distribution $p_1(m)$ over the number of features $m$, and conditional distributions $p_1(w\mid m)$ and $p_1(\theta_1, \dots, \theta_m \mid m)$ on the parameters $w \in \mathbb{R}^m$ and $\theta_1, \dots, \theta_m \in \mathbb{R}^d$ given $m$. To penalise large numbers of features, we choose $p_1(m) \propto 2^{-m}$ for $m \in \mathbb{N}$. We set both $p_1(w\mid m)$ and $p_1(\theta_1, \dots, \theta_m\mid m)$ to be uniform distributions. In particular, we set $p_1(w\mid m)$ to be the uniform distribution $\mathcal{U}(\mathbb{B}_1^m(1))$ on the $m$-dimensional unit $\ell_1$ ball $\mathbb{B}_1^m(1)$ to reflect our assumption that $\|w^*\|_1 \leq 1$. For each $m \in \mathbb{N}$, we factorise the conditional distribution over $\theta_1, \dots, \theta_m$ as $p_1(\theta_1, \dots, \theta_m\mid m) = \prod_{i=1}^{m}p_1(\theta_i)$. We choose $p_1(\theta)$ to be the uniform distribution $\mathcal{U}(\mathbb{B}_2^d(1))$ on the $d$-dimensional unit $\ell_2$ ball $\mathbb{B}_2^d(1)$ to reflect our assumption that for each $i$, $\|\theta_i^*\|_2 \leq 1$. We then have
\begin{equation*}
p_1(w \mid m) = \tfrac{m!}{2^{-m}}\mathbb{I}\{\|w\|_1 \leq 1\}\,
\quad\text{and}\quad\,
p_1(\theta) = \tfrac{\Gamma(d/2 + 1)}{\pi^{d/2}}\mathbb{I}\{\|\theta\|_2 \leq 1\}\,.
\end{equation*}
The resulting prior over parameters $\nu$ is
\begin{equation}
p_1(\nu) = \sum_{m=1}^{\infty}p_1(m)p_1(w\mid m)p_1(\theta_1, \dots, \theta_m\mid m)\,.\label{eqn:us_prior}
\end{equation}
For every $\nu$ in the support of our prior, we have $\|w\|_1 \leq 1$, which means that the prior satisfies $\mathbb{P}_{\nu \sim p_1}[\max_{x, a}|f_{\nu}(x, a)| \leq 1] = 1$. Our regret bound for FGTS in the uncountable sparsity setting is stated in the following theorem.

\begin{theorem}\label{thm:upperbounds_us}
Consider FGTS with $\eta = 1/4$ and the prior $p_1$ defined in \eqref{eqn:us_prior}. The expected regret of FGTS with $\lambda \propto \sqrt{d\log(n)/(Kn)}$ is at most
\begin{equation*}
R_n(f^*) = \mathcal{O}(\|w^*\|_0\sqrt{Kdn\log(n)})\,.
\end{equation*}
Suppose that $s \geq \|w^*\|_0$ is a known upper bound on the sparsity. The expected regret of FGTS with $\lambda \propto \sqrt{sd\log(n)/(Kn)}$ is at most
\begin{equation*}
R_n(f^*) = \mathcal{O}(\sqrt{Ksdn\log(n)})\,.
\end{equation*}
\end{theorem}
This regret bound shows that FGTS with a suitable prior is nearly minimax optimal for the uncountable sparsity setting. In particular, the regret bound for known sparsity matches the lower bound in Theorem \ref{thm:uncountable_lb} up to a factor of $\sqrt{\log(n)}$. The proof of Theorem \ref{thm:upperbounds_us} can be found in Appendix \ref{sec:us_ub_proof}.

\section{Discussion}
\label{sec:discussion}

In this work, we studied a new class of contextual bandit problems, called sparse nonparametric contextual bandits, which captures the challenge of simultaneously learning relevant features and minimising regret. Our main goal was to establish whether it is possible to exploit sparsity to obtain a flexible and sample-efficient contextual bandit algorithm. To this end, we proved upper and lower bounds on the minimax regret for this class of problems, which match up to logarithmic factors of $n$. Our findings are mixed. On the one hand, our lower bounds have polynomial dependence on the number of actions, which suggests that it is difficult to exploit sparsity when the number of actions is large. On the other hand, we showed that the Feel-Good Thompson Sampling algorithm, with suitable sparsity priors, enjoys regret bounds with mild dependence on the number of candidate features. When applied to sparse kernelised contextual bandits and neural contextual bandits (cf. Appendix \ref{sec:kbs}), we found that sparsity always enables better regret bounds, as long as the horizon and/or the dimension of the contexts is large enough relative to the sparsity and the number of actions. In the subsections that follow, we discuss some limitations of and questions raised by our work.

\subsection{Computationally Efficient Algorithms}\label{sec:appcompeff}

An important question is whether or not there exists a computationally efficient algorithm that has a (nearly) matching regret bound. For FGTS, this question boils down to whether or not one can sample from the posterior efficiently. In particular, if there is access to an oracle that samples from the posterior, then FGTS is oracle-efficient. We are required to sample from the posterior only once per round and the rest of the algorithm has time complexity which polynomial in $K$, $s$, $n$, $d$ and the dimension of the contexts. Several works have demonstrated that Markov Chain Monte Carlo methods can be used to approximately sample from the sparse posteriors that appear in our analysis \citep{alquier2011pac, rigollet2011ewa, guedj2013pac}. However, we are not aware of any way to sample from the posteriors in Section \ref{sec:ubs} with polynomial time complexity.

Several oracle-based contextual bandit algorithms, such as SquareCB \citep{foster2020beyond}, are known to lead to sample-efficient and computationally efficient algorithms whenever the oracle is both sample-efficient and computationally efficient. For SquareCB in particular, this reduces the problem to that of finding a sample-efficient and computationally efficient algorithm for sparse nonparametric regression. However, it is known that even sparse linear regression exhibits an unfortunate trade-off between sample and computational efficiency \citep{zhang2014lower}. If we require a polynomial-time algorithm, we could resort to using ``slow-rate'' bounds for the LASSO (see e.g.\ Section 7.4 in \citet{wainwright2019high}) for finite-dimensional or countable sparse regression, or the Beurling LASSO \citep{bach2017breaking} for uncountable sparse regression. However, the regret bound of SquareCB would have a sub-optimal $n^{3/4}$ rate. Moreover, it is not clear that this actually would lead to a polynomial-time algorithm for sparse nonparametric regression. The time complexity of the LASSO is polynomial in the number of candidate features, but in the countable sparsity model, the (effective) number of features is $\deff$, which can be exponential in the dimension of the contexts. Similarly, it is not known whether the Beurling LASSO can be computed with polynomial time complexity.

\subsection{Large Action Sets}\label{sec:largeactionsets}

Our lower bounds show that under our mild set of assumptions it is not possible to have low worst-case regret when the set of actions is large. It would be interesting to investigate which additional assumptions can lead to regret bounds with improved dependence on the number of actions. In sparse linear contextual bandits with contexts drawn i.i.d.\ from a well-conditioned distribution, there are several algorithms that achieve logarithmic (or better) dependence on both the dimension and the number of actions \citep{kim2019doubly, oh2021sparsity}. Loosely speaking, the context distribution is well-conditioned if there is low correlation between the features. This notion of a well-conditioned distribution can be made precise using conditions studied in high-dimensional statistics, such as the compatibility condition \citep{buhlmann2011statistics}, restricted eigenvalue conditions \citep{bickel2009simultaneous} and the restricted isometry property \citep{candes2005decoding, candes2007dantzig} to name a few. Unfortunately, when the number of features is infinite, these conditions can fail to hold at all. In the field of compressed sensing, a number of weaker conditions, known as individual or non-uniform recovery conditions (cf.\ Section 4.4 in \citet{foucart2013mathematical}), have been studied and applied to sparse infinite-dimensional linear models. In the uncountable sparsity model, these conditions roughly state that recovery is possible if $\theta_1^*, \dots, \theta_s^*$ are well-separated \citep{candes2014towards, duval2015exact, azais2015spike, poon2023geometry}. One could investigate whether these individual recovery conditions enable regret bounds for sparse nonparametric contextual bandits with logarithmic (or better) dependence on the number of actions.

These conditions on the context distribution also have implications for designing computationally efficient algorithms. If the restricted isometry property is satisfied, then it is possible to sample from spike-and-slab posteriors for sparse (finite-dimensional) linear models with polynomial time complexity \citep{kumar2025spike}. If one could prove a similar result for sparse nonparametric models, then (under suitable conditions on the context distribution) it would be possible to run FGTS with polynomial time complexity. Similarly, if a restricted eigenvalue condition is satisfied, then the LASSO enjoys fast-rate bounds for sparse regression (cf.\ Theorem 7.20 in \citet{wainwright2019high}), and so SquareCB would have a $\sqrt{n}$ regret bound. However, as previously mentioned, it is not clear that the LASSO (or BLASSO) has polynomial time complexity for sparse nonparametric models.

\subsection{Instance-Dependent Analysis}

We could also consider conditions on the context and/or reward distributions that allow for improved dependence of the regret on $n$. Several sparse linear bandit algorithms enjoy improved regret bounds when a margin condition (cf.\ Assumption 1 in \citet{li2021regret}) is satisfied. The margin condition roughly states that, with high probability over the random draw of the context, the gap between the reward of the best action and that of the second best action is large. In the extreme case where, with probability 1, there is a strictly positive gap, regret bounds of order $\log(n)$ can be achieved \citep{wang2018minimax, bastani2020online, li2021regret}. It would be interesting to investigate whether the margin condition enables logarithmic regret bounds for sparse nonparametric contextual bandits.

Similarly, one could investigate first-order bounds, where the dependence of the regret on $n$ is replaced with the cumulative loss (negation of the reward) of the best policy, which will be much smaller when, for instance, a margin condition is satisfied and/or the rewards are (almost) noiseless. The first efficient algorithm with first-order guarantees for contextual bandits was introduced by \cite{foster2021efficient}, and is based on a reduction from contextual bandits to online regression with cross-entropy loss. More recently, the Optimistic Information Directed Sampling (OIDS) algorithm was shown to achieve a first-order regret bound for a general class of contextual bandit problems \citep*{neu2024optimisic}. As noted by \citet*{neu2024optimisic}, their information-theoretic regret analysis is closely related to the decoupling technique proposed by \citet{zhang2022feel}, which we used in this paper. In particular, both approaches reduce the problem of bounding the regret to that of bounding the log partition function $Z_n$ in \eqref{eq:Zt}. We therefore expect that our bounds on $Z_n$ could be used to show that OIDS satisfies a first-order regret bound for sparse nonparametric contextual bandits.

\subsection{Misspecification}\label{sec:appmiss}

For simplicity, and to keep our attention focused solely on the challenge of exploiting sparsity, we assumed that the model is well-specified. To make the model more realistic, we could assume that the reward function is approximated well by a sparse function. For FGTS, achieving a good balance between the regret suffered while estimating the best $s$-sparse approximation of the reward function and the regret suffered due to approximation error likely requires the tuning parameter $\lambda$ to be set according to the misspecification level, which will decrease as $s$ increases, but the rate at which it does so may be unknown. One could investigate whether an aggregation (or ``corralling'') procedure \citep{agarwal2017corralling, foster2020adapting} could be used to tune $\lambda$ adaptively.

\acks{We would like to thank Gergely Neu, Eugenio Clerico and Ludovic Schwartz for insightful discussions that took place during the preparation of this manuscript. Hamish was funded by the European Research Council (ERC), under the European Union’s Horizon 2020 research and innovation programme (grant agreement 950180). Julia was supported by the Netherlands Organization for Scientific Research (NWO) under grant number VI.Veni.232.068. Paul was supported by the Department of Research and Universities of the Government of Catalunya and the European Social Fund.}


\bibliography{main}

\newpage

\appendix

\section{Related Work}
\label{sec:related_work}

\subsection{Sparse Models}

The sparsity priors that we use take inspiration from a line of work on exponentially weighted aggregates, which have previously been used to derive oracle inequalities for sparse linear regression \citep{leung2006information, dalalyan2008aggregation, dalalyanLangevin, dalalyan2012mirror, alquier2011pac} and regret bounds for online sparse linear regression \citep{gerchinovitz2011sparsity}. Several works have applied traditional approaches for sparse linear regression, such as the LASSO \citep{tibshirani1996regression}, Basis Pursuit \citep{chen2001atomic} and thresholding, to nonparametric regression settings that resemble the countable sparsity model that we consider. Examples include wavelet regression \citep{donoho1994ideal, donoho1995adapting, donoho1998minimax, donoho1995wavelet} and sparse additive models \citep{lin2006component, koltchinskii2008sparse, ravikumar2009sparse, raskutti2012minimax}. The uncountable sparsity model has been studied in supervised learning settings, under the name of convex neural networks \citep{bengio2005convex, rosset2007l1, bach2017breaking}. Several previous works have developed sparse kernel methods, such as the Support Vector Machine \citep{cortes1995support, vapnik1997support}, the Relevance Vector Machine \citep{tippingrvm} and sparse Gaussian processes \citep{smola2000sparse, williams2000using}. Many of these approaches use sparse estimators for the purpose of improved computational efficiency, though some sparse kernel methods exploit sparsity for improved sample efficiency \citep{shi2019sparse}.

\subsection{Sparse Linear Bandits}

The challenge of simultaneously learning relevant features and minimising regret in contextual bandits has been partially addressed by some previous works on sparse linear contextual bandits. Several works consider sparse linear contextual bandits, with $p$-dimensional feature vectors, which can be thought of as a finite-dimensional version of the problem that we consider. \citet{abbasi2012online} used an online-to-confidence-set conversion to develop an algorithm with a regret bound of order $\sqrt{spn}$.  When the contexts are chosen by an adversary and the number of actions $K$ is allowed to be large or infinite, there is a matching lower bound \citep{lattimore2020bandit}. Because of this negative result, it has become popular to study sparse linear contextual bandits with i.i.d.\ contexts drawn from a well-conditioned distribution. Here, well-conditioned means that, with high probability, the empirical context covariance matrix satisfies a compatibility condition \citep{buhlmann2011statistics}, or one of
a number of similar conditions studied in high-dimensional statistics. Under these conditions, various methods enjoy regret bounds that depend only logarithmically on the dimension $p$ \citep{foster2018practical, bastani2020online, wang2018minimax, kim2019doubly, oh2021sparsity, ariu2022thresholded, chakraborty2023thompson}. In non-contextual sparse linear bandits, \citet{hao2020high} and \citet{jang2022popart} showed that the existence of a policy that collects well-conditioned data enables problem-dependent regret bounds with logarithmic dependence on $p$. However, these approaches cannot easily be extended to sparse nonparametric bandits, because the compatibility condition (as well as other similar conditions) can fail to hold at all in both the countable and uncountable sparsity models.

\subsection{Contextual Bandits and Feature Selection}

Several works discovered conditions on the context distribution and the feature representation under which optimistic \citep{hao2020adaptive, wu2020stochastic} or greedy \citep{bastani2021mostly, kim2024local} contextual bandit algorithms are guaranteed to have constant or logarithmic (in $n$) regret. Subsequently, contextual bandit algorithms have been developed to identify these good feature representations while simultaneously minimising regret \citep{papini2021leveraging, tirinzoni2022scalable}. These results are complementary to our own, which hold when the contexts are selected by an adversary.

\subsection{Bandits With Low-Dimensional Structure}

Various methods have been developed to exploit other forms of low-dimensional structure in bandit and Bayesian optimisation problems. \citet{chen2012joint}, \citet{djolonga2013high} and \citet{wang2016bayesian} designed Bayesian optimisation algorithms for the setting where the reward function is a composition of a low-dimensional linear embedding and a function drawn from a Gaussian process. \citet{li2023dimension} designed a contextual bandit algorithm for a similar setting, in which the reward function is a composition of a linear embedding and a Lipschitz function. \citet{kandasamy2015high}, \citet{gardner2017discovering}, \citet{rolland2018high} and \citet{mutny2018efficient} developed Bayesian optimisation algorithms for (generalised) additive reward functions.

\subsection{Compressed Sensing Off the Grid}

The nonparametric contextual bandit problem with uncountable sparsity is related to a problem known as compressed sensing off the grid \citep{tang2013compressed}. This is also called an inverse problem in the space of measures \citep{bredies2013inverse}, or (blind) deconvolution when applied to signals \citep{levin2009understanding}. In this problem, the aim is to recover the underlying discrete measure in a sparse infinite-dimensional linear model, using as few measurements as possible. In the finite-dimensional compressed sensing problem, conditions on the measurements similar to the compatibility condition used in sparse linear bandits are sufficient to guarantee recovery via, for instance, the LASSO (see e.g.\ chapters 4-6 in \citet{foucart2013mathematical}). To guarantee recovery in the off-the-grid compressed sensing problem, several works have developed weaker and more refined conditions on the measurements \citep{duval2015exact, bodmann2018compressed, poon2023geometry}. These conditions ensure that the underlying sparse measure can be recovered using an infinite-dimensional formulation of the LASSO, known as the Beurling LASSO \citep{de2012exact}.

\section{Application to Linear, Kernelised and Neural Bandits}

\subsection{Implications for Sparse Linear Contextual Bandits}

Our results have interesting implications for particular cases of sparse linear contextual bandits. With a small modification, Theorem \ref{thm:countable_lb} also provides a lower bound of $\sqrt{Ksn}/8$ for sparse linear contextual bandits with $p$-dimensional feature vectors and $K \leq p/s$. This complements the existing lower bound of order $\sqrt{spn}$ \citep{lattimore2020bandit}, in which the difficult instance has $K \geq p$. When $K \leq p/s$, our Theorem \ref{thm:upperbounds_cs} gives a matching upper bound up to a factor of $\sqrt{\log(n)}$.

In linear contextual bandits with adversarial contexts, $p$-dimensional feature vectors and $K$ actions, SupLinRel \citep{auer2002using} and SupLinUCB \citep{chu2011contextual} both have regret bounds of order approximately $\sqrt{pn\log(K)}$. If the parameter vector is $s$-sparse (and $s$ is known), then as discussed in Section \ref{sec:related_work}, there is an algorithm that achieves an upper bound of $\sqrt{spn}$, and a matching lower bound. This suggests that sparsity is helpful when the number of actions is large, relative to the sparsity, but not particularly helpful when $p$ is large. When applied to this setting, where $\deff = p$, Theorem \ref{thm:upperbounds_cs} would give a regret bound of order $\sqrt{Ks\log(p)n}$. Roughly speaking, this is an improvement whenever the number of actions is less than the ratio of the total number of features and the number of useful features, so when $K \leq p/s$.

\subsection{When Does Sparsity Help In Kernelised and Neural Bandits?}
\label{sec:kbs}


In this section, we use our upper and lower bounds from sections \ref{sec:lbs} and \ref{sec:ubs} to identify regimes in which sparsity is helpful for regret minimisation in 
kernelised and neural contextual bandits. For simplicity, we ignore polylogarithmic factors of $n$ in this section. Our findings are summarised in Table \ref{tab:matern}, for kernelised bandits with the Mat\'{e}rn kernel, Table \ref{tab:rbf}, for kernelised bandits with the RBF kernel, and Table \ref{tab:neural}, for neural bandits.

\begin{table}[H]
\centering
\begin{tabular}{lcc}
\label{tab:matern}
Upper Bound & Regret & Sparsity Assumption\\
\hline
\citet{valko2013finite} & $\mathcal{O}\big(n^{\frac{\nu + p}{2\nu + p}}\sqrt{\log(K)}\big)$ & none\\
This paper (Theorem \ref{thm:upperbounds_cs}) & $\mathcal{O}(\sqrt{Ksn\log(n)})$ & countable\\
This paper (Theorem \ref{thm:upperbounds_us}) & $\mathcal{O}(\sqrt{Kspn})$ & uncountable\\
\-\ & \-\ & \-\ \\
Lower Bound & \-\ & \-\ \\
\hline
This paper (Theorem \ref{thm:countable_lb}) & $\Omega(\sqrt{Ksn})$ & countable\\
This paper (Theorem \ref{thm:uncountable_lb}) & $\Omega(\sqrt{Kspn})$ & uncountable\\
\end{tabular}
\vspace{1mm}
\caption{Regret upper and lower bounds for kernelised contextual bandits with the Mat\'{e}rn kernel.}
\end{table}

In kernelised contextual bandits with adversarial contexts (of dimension $p$) and $K$ actions, SupKernelUCB \citep{valko2013finite} has a regret bound of order $\sqrt{\deff n\log(K)}$, where $\deff$ depends on the choice of the kernel. If the Mat\'{e}rn kernel with smoothness parameter $\nu > 0$ is used, then the polynomial decay condition is satisfied with $\beta = (2\nu + p)/p$ \citep{santin2016approximation}, and so $\deff \leq n^{p/(2\nu + p)}$. If the RBF kernel is used, then the exponential decay condition is satisfied with $\beta = 1/p$ \citep{belkin2018approximation}, and so $\deff \leq \log^p(n)$. In the countable sparsity model, Theorem \ref{thm:upperbounds_cs} gives a regret bound of order $\sqrt{Ksn\log(\deff)}$. For the Mat\'{e}rn kernel, this translates to a bound of order $\sqrt{(p/(2\nu + p))Ksn\log(n)}$, which can be compared to the bound of $n^{\frac{\nu + p}{2\nu + p}}\sqrt{\log(K)}$ for SupKernelUCB. For the RBF kernel, the bound from Theorem \ref{thm:upperbounds_cs} becomes $\sqrt{Kspn\log\log(n)}$, whereas the bound for SupKernelUCB is $\sqrt{n\log^{p}(n)\log(K)}$. With both of these kernels, we observe that when $n$ is sufficiently large, approximately $n \geq (sK)^{(2\nu + p)/p}$ for the Mat\'{e}rn kernel and $n \geq \exp((Ksp)^{1/p})$ for the RBF kernel, sparsity enables better regret bounds. In the uncountable sparsity model, the dimension of $\theta$ is $d = p$. Theorem \ref{thm:upperbounds_us} gives a regret bound of order $\sqrt{Kspn}$ with both the Mat\'{e}rn and RBF kernels. Again, we observe that for sufficiently large $n$, approximately $n \geq (Ksp)^{(2\nu + p)/p}$ for the Mat\'{e}rn kernel and $n \geq \exp((Ksp)^{1/p})$ for the RBF kernel, sparsity enables better regret bounds.

\begin{table}
\label{tab:rbf}
\centering
\begin{tabular}{lcc}
Upper Bound & Regret & Sparsity Assumption\\
\hline
\citet{valko2013finite} & $\mathcal{O}\big(\sqrt{n\log^{p}(n)\log(K)}\big)$ & none\\
This paper (Theorem \ref{thm:upperbounds_cs}) & $\mathcal{O}(\sqrt{Kspn\log\log(n)})$ & countable\\
This paper (Theorem \ref{thm:upperbounds_us}) & $\mathcal{O}(\sqrt{Kspn})$ & uncountable\\
\-\ & \-\ & \-\ \\
Lower Bound & \-\ & \-\ \\
\hline
This paper (Theorem \ref{thm:countable_lb}) & $\Omega(\sqrt{Kspn})$ & countable\\
This paper (Theorem \ref{thm:uncountable_lb}) & $\Omega(\sqrt{Kspn})$ & uncountable\\
\end{tabular}
\vspace{1mm}
\caption{Regret upper and lower bounds for kernelised contextual bandits with the RBF kernel.}
\end{table}

In neural contextual bandits, it is often assumed that the reward function lies in the RKHS associated with an infinite-width neural tangent kernel (NTK). In this setting, with adversarial contexts and $K$ actions, the SupNN-UCB algorithm \citep{kassraie2022neural} has a regret bound of order approximately $\sqrt{\deff n \log(K)}$, where $\deff$ is the effective dimension of the NTK. If the contexts are $p$-dimensional vectors on the unit sphere, then for the NTK of a single-layer ReLU neural network, $\deff \leq n^{1 - 1/p}$ (cf.\ Theorem 3.1 in \citet{kassraie2022neural}). In this setting, the regret bound of SupNN-UCB is of order $n^{1 - 1/(2p)}\sqrt{\log(K)}$. In the neural bandits example in Section \ref{sec:examples}, the reward function is assumed to be a single-layer neural network of possibly unknown width. If we use a ReLU activation function, then Theorem \ref{thm:upperbounds_us} gives a regret bound of order $\sqrt{Kspn}$. For sufficiently large $n$, approximately $n \geq (Ksp)^{1-1/p}$, sparsity enables a better regret bound.

\begin{table}[H]
\label{tab:neural}
\centering
\begin{tabular}{lcc}
Upper Bound & Regret & Sparsity Assumption\\
\hline
\citet{kassraie2022neural} & $\mathcal{O}\big(n^{1 - \frac{1}{2p}}\sqrt{\log(K)}\big)$ & none\\
This paper (Theorem \ref{thm:upperbounds_us}) & $\mathcal{O}(\sqrt{Kspn})$ & uncountable\\
\-\ & \-\ & \-\ \\
Lower Bound & \-\ & \-\ \\
\hline
This paper (Theorem \ref{thm:uncountable_lb}) & $\Omega(\sqrt{Kspn})$ & uncountable\\
\end{tabular}
\vspace{1mm}
\caption{Regret upper and lower bounds for neural contextual bandits.}
\end{table}

\section{Proofs of the Lower Bounds}
\label{app:lower_bound_proof}

The proofs of Theorem \ref{thm:countable_lb} and Theorem \ref{thm:uncountable_lb} both work by first reducing the sparse nonparametric contextual bandit problem to a sequence of $K$-armed bandits, and then lower bounding the total regret in the sequence of $K$-armed bandit problems. Each type of sparsity requires a different reduction. However, both proofs use the same auxiliary lemmas to lower bound the regret in the sequence of $K$-armed bandit problems. In Section \ref{sec:lb_aux}, we state and prove these auxiliary lemmas. In the subsequent subsections, we use these lemmas to prove Theorem \ref{thm:countable_lb} and then Theorem \ref{thm:uncountable_lb}.

\subsection{Auxiliary Lemmas}
\label{sec:lb_aux}
The first auxiliary lemma is a lower bound on the regret for $K$-armed bandit problems in which one action has mean reward $\Delta$ and the remaining $K-1$ actions have mean reward 0. The rewards of all actions are subject to standard Gaussian noise. For some $b \in [K]$, let $R_m(b)$ denote the expected regret over $m$ steps in a $K$-armed bandit with optimal action $b$. Therefore,
\begin{equation*}
R_m(b) = \mathbb{E}\left[\sum_{t=1}^{m}\Delta\cdot\mathbb{I}\{A_t \neq b\}\right].
\end{equation*}

Exercise 24.1 in \citet{lattimore2020bandit} asks the reader to prove the following lower bound on the averaged (over $b$) expected regret $\frac{1}{K}\sum_{b \in [K]}R_m(b)$.

\begin{lemma}[Exercise 24.1 in \citet{lattimore2020bandit}]
Let $K \geq 2$ and set $\Delta = \sqrt{K/m}/2$. For any policy,
\begin{equation*}
\frac{1}{K}\sum_{b \in [K]}R_m(b) \geq \frac{1}{8}\sqrt{Km}.
\end{equation*}
\label{lem:K_armed_lb}
\end{lemma}

A proof can be found in \citet{lattimore2020solutions}. To use this result, we need to reduce the sparse nonparametric contextual bandit problem to a sequence of $K$-armed bandit problems. In the proofs of each lower bound, we always factorise the horizon as $n = m_1m_2$ for some $m_1,m_2\in \mathbb{N}$ and fix the sequence of contexts $x_1, \dots, x_n$ to be $z_1, \dots, z_1, z_2, \dots, z_2, \dots, z_{m_1}, \dots, z_{m_1}$, where each $z_i$ is repeated $m_2$ times. For each lower bound, we choose a set of parameters $N$ of size $|N| = K^{m_1}$ and a feature map $\phi$, such that for every $i \in [m_1]$, $a \in [K]$ and $\nu \in N$, we can express the expected reward as
\begin{equation}
f_{\nu}(z_i, a) = \tfrac{\Delta}{s}\mathbb{I}\{a = b_i(\nu)\}\,.\label{eqn:seq_reward}
\end{equation}

The function $b: N \to [K]^{m_1}$ is any bijection between $N$ and the set of sequences of actions of length $m_1$. This reward function mimics a sequence of $m_1$ $K$-armed bandit problems, in which the index of the optimal action switches every $m_2$ rounds. For each context $z_i$ and each $\nu \in N$, there is a single good action $b_i(\nu)$ with expected reward $\Delta/s$ and $K-1$ bad actions with expected reward $0$. Since each $\nu \in N$ corresponds to a unique sequence $(b_1, \dots, b_{m_1}) \in [K]^{m_1}$ of optimal actions (and vice versa), we can equivalently parameterise the reward function by $b_{1:m_1} := (b_1, \dots, b_{m_1}) \in [K]^{m_1}$. In particular, for each $b_{1:m_1} \in [K]^{m_1}$, we can write the expected regret $R_n(b_{1:m_1})$ as
\begin{equation*}
R_n(b_{1:m_1}) := \mathbb{E}\left[\sum_{t=1}^{n}\sum_{i=1}^{m_1}\tfrac{\Delta}{s}\mathbb{I}\{x_t = z_i\}\mathbb{I}\{A_t \neq b_i\}\right]\,.
\end{equation*}

The following lemma shows that, for any $b_{1:m_1} \in [K]^{m_1}$, $R_n(b_{1:m_1})$ can be re-written as the total expected regret suffered in a sequence of $K$-armed bandit problems.\medskip

\begin{lemma}[Regret decomposition]
For any $b_{1:m_1} \in [K]^{m_1}$,
\begin{equation*}
R_n(b_{1:m_1}) = \sum_{i=1}^{m_1}R_{m_2, i}(b_{1:i})\,,
\end{equation*}
where
\begin{equation*}
R_{m_2, i}(b_{1:i}) = \mathbb{E}\left[\sum_{t=m_2(i-1)+1}^{m_2i}\tfrac{\Delta}{s}\mathbb{I}\{A_t \neq b_i\}\right].
\end{equation*}
\label{lem:regret_decomp}
\end{lemma}

\begin{proof}
The proof only requires us to swap a sum and an expectation. In particular,
\begin{align*}
R_n(b_{1:m_1}) &= \mathbb{E}\left[\sum_{t=1}^{n}\sum_{i=1}^{m_1}\tfrac{\Delta}{s}\mathbb{I}\{x_t = z_i\}\mathbb{I}\{A_t \neq b_i\}\right]\\
&= \sum_{i=1}^{m_1}\mathbb{E}\left[\sum_{t=1}^{n}\tfrac{\Delta}{s}\mathbb{I}\{x_t = z_i\}\mathbb{I}\{A_t \neq b_i\}\right]\\
&= \sum_{i=1}^{m_1}\mathbb{E}\left[\sum_{t=m_2(i-1)+1}^{m_2i}\tfrac{\Delta}{s}\mathbb{I}\{A_t \neq b_i\}\right] = \sum_{i=1}^{m_1}R_{m_2, i}(b_{1:i}).
\end{align*}
\end{proof}

We notice that $R_{m_2,i}(b_{1:i})$ is the regret suffered in $m_2$ steps of a $K$-armed bandit problem, in which the optimal action is $b_i$ and all other actions result in an expected regret of $\Delta/s$. We write $R_{m_2, i}(b_{1:i})$ (as opposed to $R_{m_2, i}(b_{1:m_1})$) because the expected regret for the $i$\textsuperscript{th} sub-problem is independent of $b_{i+1:m_1}$. $R_{m_2,i}(b_{1:i})$ can depend on the first $i-1$ elements $b_{1:i-1}$ of the sequence $b_{1:m_1}$, since the rewards obtained in the first $(i-1)m_2$ rounds can influence the policy played in rounds $t = (i-1)m_2 + 1$ to $t = im_2$. However, the lower bound in Lemma \ref{lem:K_armed_lb} applies to any policy, and hence any $b_{1:i-1}$. The final auxiliary lemma combines the previous two, and provides a lower bound on the averaged regret $\frac{1}{K^{m_1}}\sum_{b_{1:m_1} \in [K]^{m_1}}R_n(b_{1:m_1})$.\medskip

\begin{lemma}[Lower bound for sequences of $K$-armed bandits]
Let $K \geq 2$ and set $\Delta = s\sqrt{K}/\sqrt{4m_2}$. For any $b_{1:m_1} \in [K]^{m_1}$ and any policy,
\begin{equation*}
\frac{1}{K^{m_1}}\sum_{b_{1:m_1} \in [K]^{m_1}}R_n(b_{1:m_1}) \geq \frac{1}{8}m_1\sqrt{Km_2}\,.
\end{equation*}\label{lem:seq_lb}
\end{lemma}

\begin{proof}
Choosing $\Delta = s\sqrt{K}/\sqrt{Km_2}$ ensures that $\Delta/s = \sqrt{K}/\sqrt{4m_2}$, which is required for the lower bound in Lemma \ref{lem:K_armed_lb}. Using Lemma \ref{lem:K_armed_lb}, and the fact that $R_{m_2, i}(b_{1:i})$ is independent of $b_{i+1:m_1}$, we obtain,
\begin{align*}
\frac{1}{K^{m_1}}\sum_{b_{1:m_1} \in [K]^{m_1}}R_n(b_{1:m_1}) &= \sum_{i=1}^{m_1}\frac{1}{K^{m_1}}\sum_{b_{1:m_1} \in [K]^{m_1}}R_{m_2,i}(b_{1:m_1})\\
&= \sum_{i=1}^{m_1}\frac{1}{K^{i-1}}\sum_{b_{1:i-1} \in [K]^{i-1}}\frac{1}{K}\sum_{b_{i} \in [K]}\frac{1}{K^{m_1-i}}\sum_{b_{i+1:m_1} \in [K]^{m_1-i}}R_{m_2,i}(b_{1:i})\\
&= \sum_{i=1}^{m_1}\frac{1}{K^{i-1}}\sum_{b_{1:i-1} \in [K]^{i-1}}\frac{1}{K}\sum_{b_{i} \in [K]}R_{m_2,i}(b_{1:i})\\
&\geq \frac{1}{8}\sum_{i=1}^{m_1}\frac{1}{K^{i-1}}\sum_{b_{1:i-1} \in [K]^{i-1}}\sqrt{Km_2}\\
&= \frac{1}{8}m_1\sqrt{Km_2}\,.
\end{align*}
\end{proof}

\subsection{Proof of Theorem \ref{thm:countable_lb}}
\label{sec:count_lb_proof}

For the convenience of the reader, we repeat the statement of Theorem \ref{thm:countable_lb} here.

\textit{Statement of Theorem \ref{thm:countable_lb}.} Consider the sparse nonparametric contextual bandit problem with countable sparsity described in Section \ref{sec:setting}. Let $\mathcal{A} = [K]$ for some $K \geq 2$ and assume that the noise variables are standard Gaussian. Suppose that for some $\beta > 1$ and some integer $m \geq s^{\beta + 2}K^{\beta + 1}$, $n = sm$. Then for any policy, there exists a sequence of contexts $x_1, \dots, x_n$, a parameter sequence $w = (w_1, w_2, \dots)$ with $\|w\|_0 = s$, $\|w\|_1 = 1$ and a sequence of functions $(\phi_i)_{i=1}^{\infty}$ with $\|\phi_i\|_{\infty} \leq i^{-\beta/2}$, such that
\begin{equation*}
R_n(f_{w}) \geq \frac{1}{8}\sqrt{Ksn}.
\end{equation*}

Instead, suppose that for some $\beta > 0$ and some integer $m \geq \lceil 1/\beta\rceil s^2K\exp(s^{\beta}K^{\beta\lceil 1/\beta\rceil})$, $n = sm$. Then for any policy, there exists a sequence of contexts $x_1, \dots, x_n$, a parameter sequence $w = (w_1, w_2, \dots)$ with $\|w\|_0 = s$, $\|w\|_1 = 1$ and a sequence of functions $(\phi_i)_{i=1}^{\infty}$ with $\|\phi_i\|_{\infty} \leq \exp(-i^{\beta}/2)$, such that
\begin{equation*}
R_n(f_{w}) \geq \frac{1}{8}\sqrt{\max(1, 1/\beta)Ksn}.
\end{equation*}

\begin{proof}
First, we prove the lower bound for the scenario with polynomial decay, in which $\|\phi_i\|_{\infty} \leq i^{-\beta/2}$ for some $\beta > 1$. The sequence of contexts $x_1, \dots, x_n$ is fixed in advance, and selected from the set $\{z_1, \dots, z_{s}\}$, where each $z_i$ is of the form $z_i = (z_{i,a})_{a \in [K]}$. For each $i \in [s]$ and $a \in [K]$, $z_{i,a}$ can be arbitrary, as long as each $z_{i,a}$ is distinct. We choose the sequence of contexts $x_1, \dots, x_n$ to be
\begin{equation*}
z_1, \dots, z_1, z_2, \dots, z_2, \dots, z_{s}, \dots, z_{s},
\end{equation*}

where each $z_i$ is repeated $m$ times. Next, we define the sequence of functions $(\phi_j)_{j=1}^{\infty}$. In fact, we will consider a finite sequence of functions $(\phi_j)_{j=1}^{sK}$. We could extend this to an infinite sequence by choosing $\phi_j \equiv 0$ for $j > sK$ and nothing would change. We define $e_1, \dots, e_{sK}$ to be the standard basis vectors of $\mathbb{R}^{sK}$. For each $i \in [s]$ and $a \in [K]$, and some $\Delta \in (0, 1]$ to be chosen later, we define
\begin{equation*}
\phi(z_{i,a}) = (\phi_1(z_{i,a}), \dots, \phi_{sK}(z_{i,a})) = \Delta \cdot e_{(i-1)K + a}\,.
\end{equation*}

We could also write this as $\phi_j(z_{i,a}) = \Delta \cdot \mathbb{I}\{j = (i-1)K + a\}$. Note that if $\Delta \leq (sK)^{-\beta/2}$, then $\|\phi_j\|_{\infty} \leq j^{-\beta/2}$ is satisfied for all $j \in [sK]$. Next, we define a set of $s$-sparse parameter sequences. In fact, we will consider $sK$-dimensional parameter vectors, but if we wished to use an infinite sequence, we could just append zeros to the end. We let $u_1, \dots, u_{K}$ denote the standard basis vectors of $\mathbb{R}^{K}$, and we define the set $\mathcal{W} = \{(1/s)u_i: i \in [K]\} \subset \mathbb{R}^{K}$. We consider parameter vectors in the set $\mathcal{W}^s \subset \mathbb{R}^{sK}$, meaning each $w$ is of the form
\begin{equation*}
w = [(1/s)u_{i_1}^{\top}, \dots, (1/s)u_{i_s}^{\top}]^{\top}\,,
\end{equation*}

for some collection of indices $(i_1, \dots, i_{s}) \in [K]^s$. We define the bijection $b: \mathcal{W}^s \to [K]^s$ to be the function that maps each $w$ to its corresponding sequence of indices $(i_1, \dots, i_s)$. Thus for any $i \in [s]$, $b_i(w)$ is the position of the non-zero element within the $i$\textsuperscript{th} block of $w$. We notice that each $w \in \mathcal{W}^s$ satisfies $\|w\|_0 = s$ and $\|w\|_1 = 1$. Moreover, for any $i \in [s]$, $a \in [K]$ and $w \in \mathcal{W}^s$ we can write down the expected reward as
\begin{equation*}
f_{w}(z_i, a) = \sum_{j=1}^{sK}w_j\phi_j(z_{i,a}) = \tfrac{\Delta}{s}\mathbb{I}\{a = b_i(w)\}\,.
\end{equation*}

This reward function (and this sequence of contexts) mimics a sequence of $K$-armed bandit problems, in which the index of the optimal action switches every $m$ rounds. In particular, it is of the same form as the reward function in (\ref{eqn:seq_reward}). Since each $w \in \mathcal{W}^s$ corresponds to a unique sequence $(b_1, \dots, b_{s}) \in [K]^{s}$ of optimal actions (and vice versa), we can equivalently parameterise the reward function by the sequence $b_{1:s} := (b_1, \dots, b_s) \in [K]^s$. Now, using Lemma \ref{lem:regret_decomp}, we can re-write the regret $R_n(b_{1:s})$ as
\begin{equation*}
R_n(b_{1:s}) = \sum_{i=1}^{s}R_{m, i}(b_{1:i})\,,
\end{equation*}

where $R_{m, i}(b_{1:i}) = \mathbb{E}[\sum_{t=m(i-1)+1}^{mi}\frac{\Delta}{s}\mathbb{I}\{A_t \neq b_i\}]$ is the expected regret for the $i$\textsuperscript{th} sub-problem. Using this regret decomposition and Lemma \ref{lem:seq_lb} (and the fact that $n = sm$), we can lower bound the averaged expected regret as
\begin{equation*}
\frac{1}{K^s}\sum_{b_{1:s} \in [K]^{s}}R_n(b_{1:s}) \geq \frac{1}{8}s\sqrt{Km} = \frac{1}{8}\sqrt{Ksn}\,.
\end{equation*}

Lemma \ref{lem:seq_lb} requires us to set $\Delta$ such that $\Delta/s = \sqrt{K/m}/2$. The only thing left to do is to find the values of $m$ such that $\Delta = s\sqrt{\frac{K}{4m}} \leq (sK)^{-\beta/2}$, which ensures that $\phi_1, \dots, \phi_{sK}$ satisfies the polynomial decay condition. This constraint for $m$ can be rearranged into
\begin{equation*}
m \geq \tfrac{1}{4}s^{\beta + 2}K^{\beta + 1}\,.
\end{equation*}

Thus we can choose any $m$ satisfying $m \geq s^{\beta + 2}K^{\beta + 1}$. Next, we prove the lower bound for the scenario with exponential decay, in which $\|\phi_i\|_{\infty} \leq \exp(-i^{\beta}/2)$, for some $\beta > 0$. First, we consider the case where $\beta \geq 1$. In this case, $\max(1, 1/\beta) = 1$, so we need to lower bound the expected regret by $\sqrt{Ksn}/8$. To do so, we can mostly repeat the proof for polynomial decay. We set $(\phi_j)_{j=1}^{\infty}$ to be the same sequence of functions, which satisfies $\|\phi_j\|_{\infty} \leq \Delta$ for $j \leq sK$ and $\|\phi_j\|_{\infty} = 0$ for $j \geq sK$. If $\Delta \leq \exp(-s^{\beta}K^{\beta}/2)$, then $(\phi_j)_{j=1}^{\infty}$ satisfies the exponential decay condition. We can factorise $n$ as $n = sm$ and repeat the rest of the proof to obtain the desired lower bound of $\sqrt{Ksn}/8$. In order to use Lemma \ref{lem:seq_lb}, we again require $\Delta = s\sqrt{K}/\sqrt{4m}$. This means that the exponential decay condition is satisfied when $m$ satisfies the inequality $s\sqrt{K}/\sqrt{4m} \leq \exp(-s^{\beta}K^{\beta}/2)$. This is equivalent to
\begin{equation*}
m \geq \tfrac{1}{4}s^2K\exp(s^{\beta}K^{\beta})\,.
\end{equation*}

Thus we can choose any $m \geq s^2K\exp(s^{\beta}K^{\beta})$. Due to the condition on $n$ in the statement of the theorem, when $\beta \geq 1$, we can indeed factorise $n$ as $n = sm$, for some integer $m \geq s^2K\exp(s^{\beta}K^{\beta})$.

For the case where $\beta \in (0,1)$, we require a more elaborate reduction to sequences of $K$-armed bandit problems. The proof for this case generalises the previous one in the sense that, for some positive integer $l \neq 1$, we reduce the problem to a sequence of $sl$ $K$-armed bandit problems, each with horizon $m$. We factorise the horizon as $n = s\lceil 1/\beta\rceil m$. The sequence of contexts $x_1, \dots, x_n$ is fixed in advance, and selected from the set $\{z_1, \dots, z_{s\lceil 1/\beta\rceil}\}$. For each $i \in [s\lceil 1/\beta\rceil]$ and $a \in [K]$, $z_{i,a}$ can be arbitrary, as long as each $z_{i,a}$ is distinct. We choose the sequence of contexts $x_1, \dots, x_n$ to be
\begin{equation*}
z_1, \dots, z_1, z_2, \dots, z_2, \dots, z_{s\lceil 1/\beta\rceil}, \dots, z_{s\lceil 1/\beta\rceil},
\end{equation*}

where each $z_i$ is repeated $m$ times. Next, we define the sequence of functions $(\phi_j)_{j=1}^{\infty}$. For each $i \in [s\lceil 1/\beta\rceil]$, we define the bijection $\rho:[s\lceil 1/\beta\rceil] \to [s]\times[\lceil 1/\beta\rceil]$ by
\begin{equation*}
\rho(i) = (\rho_1(i), \rho_2(i)) = (\lceil \tfrac{i}{\lceil 1/\beta\rceil}\rceil, (i-1)~\mathrm{mod}~ \lceil 1/\beta\rceil + 1).
\end{equation*}

Some values of $\rho$ are $\rho(1) = (1, 1)$, $\rho(2) = (1, 2)$, $\rho(\lceil 1/\beta\rceil) = (1, \lceil 1/\beta\rceil)$, $\rho(\lceil 1/\beta\rceil+1) = (2, 1)$, $\rho(s\lceil 1/\beta\rceil) = (s, \lceil 1/\beta\rceil)$. We let $\zeta: [K^{\lceil 1/\beta\rceil}] \to [K]^{\lceil 1/\beta\rceil}$ be any bijection between $[K^{\lceil 1/\beta\rceil}]$, the set of all positive integers up to $K^{\lceil 1/\beta\rceil}$, and $[K]^{\lceil 1/\beta\rceil}$, the set of sequences of integers in $[K]$ of length $\lceil 1/\beta\rceil$. For instance, we could choose
\begin{align*}
(\zeta_1(i), \zeta_2(i), \dots, \zeta_{\lceil 1/\beta\rceil}(i)) = (\lfloor\tfrac{i-1}{K^{\lceil 1/\beta\rceil-1}}\rfloor + 1, \lfloor\tfrac{i-1}{K^{\lceil 1/\beta\rceil-2}}\rfloor ~\mathrm{mod}~ K + 1, \dots, (i-1) ~\mathrm{mod}~ K + 1).
\end{align*}

We can now define the sequence of functions $(\phi_j)_{j=1}^{\infty}$. Similarly to before, we will consider a finite sequence of $sK^{\lceil 1/\beta\rceil}$ functions, but we could extend it to an infinite sequence by choosing $\phi_j \equiv 0$ for all $j > sK^{\lceil 1/\beta\rceil}$. For each $i \in [s\lceil 1/\beta\rceil]$, $a \in [K]$ and $j \in [sK^{\lceil 1/\beta\rceil}]$, we define the $j$\textsuperscript{th} function in the sequence to be
\begin{equation}
\phi_j(z_{i,a}) = \Delta\cdot\mathbb{I}\{\lfloor\tfrac{j-1}{K^{\lceil 1/\beta\rceil}}\rfloor + 1 = \rho_1(i)\}\cdot\mathbb{I}\{a = \zeta_{\rho_2(i)}((j-1)~\mathrm{mod}~ K^{\lceil 1/\beta\rceil} + 1)\}\,,\label{eqn:def_phi_j}
\end{equation}

where $\Delta \in (0, 1]$ is some positive constant to be chosen later. If we let $\phi(z_{i,a})$ denote the vector in $\mathbb{R}^{sK^{\lceil 1/\beta\rceil}}$ whose $j$\textsuperscript{th} element is $\phi_j(z_{i,a})$, we can describe the functions $\phi_1, \dots, \phi_{sK^{\lceil 1/\beta\rceil}}$ in a simpler way. We split the vector $\phi(z_{i,a})$ into $s$ blocks of $K^{\lceil 1/\beta\rceil}$ elements and write
\begin{equation}
\phi(z_{i, a}) = [0, \dots, 0, \Delta\cdot\mathbb{I}\{a = \zeta_{\rho_2(i)}(1)\}, \dots, \Delta\cdot\mathbb{I}\{a = \zeta_{\rho_2(i)}(K^{\lceil 1/\beta\rceil})\}, 0, \dots, 0]\,,\label{eqn:def_phi_vec}
\end{equation}

where the block of $K^{\lceil 1/\beta\rceil}$ elements that are not identically equal to zero is the $\rho_1(i)$\textsuperscript{th} block. To see that this is the same as the previous definition, we can observe that the first indicator in (\ref{eqn:def_phi_j}) sets $\phi_j(z_{i,a})$ to 0 whenever $j$ does not correspond to an index in the $\rho_1(i)$\textsuperscript{th} block of $\phi(z_{i,a})$. For $j$ in the $\rho_1(i)$\textsuperscript{th} block of $\phi(z_{i,a})$, i.e.\ $j$ satisfying
\begin{equation*}
(\rho_1(i)-1)K^{\lceil 1/\beta\rceil} + 1 \leq j \leq \rho_1(i)K^{\lceil 1/\beta\rceil}\,,
\end{equation*}

$(j-1)~\mathrm{mod}~ K^{\lceil 1/\beta\rceil} + 1$ is equal to the position of $j$ within the $\rho_1(i)$\textsuperscript{th} block, which means the second indicator in (\ref{eqn:def_phi_j}) is equal to the indicator in (\ref{eqn:def_phi_vec}) for these $j$. To satisfy the exponential decay condition, $\Delta$ must be set small enough. In particular, since $\|\phi_j\|_{\infty} \leq \Delta$ for all $j \in [sK^{\lceil 1/\beta\rceil}]$, then as long as $\Delta$ satisfies $\Delta \leq \exp(-s^{\beta}K^{\beta \lceil 1/\beta\rceil}/2)$, each $\phi_j$ satisfies $\|\phi_j\|_{\infty} \leq \exp(-j^{\beta}/2)$.

Next, we define a set of $s$-sparse parameter sequences. We will consider $s$-sparse vectors in $\mathbb{R}^{sK^{\lceil 1/\beta\rceil}}$, but we could extend this to infinite sequences by appending zeros. We let $u_1, \dots, u_{K^{\lceil 1/\beta\rceil}}$ denote the standard basis vectors of $\mathbb{R}^{K^{\lceil 1/\beta\rceil}}$, and we define the set $\mathcal{W} = \{(1/s)u_i: i \in [K^{\lceil 1/\beta\rceil}]\}$. We consider parameter vectors in the set $\mathcal{W}^s$, so each $w$ is of the form
\begin{equation*}
w = [(1/s)u_{i_1}^{\top}, \dots, (1/s)u_{i_s}^{\top}]^{\top}\,,
\end{equation*}

for some collection of indices $(i_1, \dots, i_{s}) \in [K^{\lceil 1/\beta\rceil}]^s$. We define the bijection $\omega: \mathcal{W}^s \to [K^{\lceil 1/\beta\rceil}]^s$ to be the function that maps each $w$ to it's corresponding sequence of indices $(i_1, \dots, i_s)$. Thus, for any $i \in [s]$, $\omega_i(w)$ is the position of the non-zero element within the $i$\textsuperscript{th} block of $w$. We notice that each $w \in \mathcal{W}^s$ satisfies $\|w\|_0 = s$ and $\|w\|_1 = 1$. Now, we can write down the expression for the expected reward function. In particular,
\begin{equation*}
f_{w}(z_i, a) = \sum_{j=1}^{sK^{\lceil 1/\beta\rceil}}w_j\phi_j(z_{i,a}) = \tfrac{\Delta}{s}\mathbb{I}\{a = \zeta_{\rho_2(i)}(\omega_{\rho_1(i)}(w))\}.
\end{equation*}

To simplify this expression, we define $b:\mathcal{W}^s \to [K]^{s\lceil 1/\beta\rceil}$ to be the bijection that maps each $w$ to the corresponding sequence of good actions. Thus $b(w) = (b_1(w), \dots, b_{s\lceil 1/\beta\rceil}(w))$, where for each $i \in [s\lceil 1/\beta\rceil]$,
\begin{equation*}
b_i(w) = \zeta_{\rho_2(i)}(\omega_{\rho_1(i)}(w)).
\end{equation*}

Thus, for any $i \in [s\lceil 1/\beta\rceil]$, $a \in [K]$ and $w \in \mathcal{W}^s$ we can express the expected reward as
\begin{equation*}
f_{w}(z_i, a) = \tfrac{\Delta}{s}\mathbb{I}\{a = b_i(w))\}.
\end{equation*}

We notice that we have another reward function that is in the same form as the reward function in (\ref{eqn:seq_reward}). Since $b$ is a bijection between $\mathcal{W}^s$ and $[K]^{s\lceil 1/\beta\rceil}$, each $w \in \mathcal{W}^s$ corresponds to a unique sequence $(b_1, \dots, b_{s\lceil 1/\beta\rceil}) \in [K]^{s\lceil 1/\beta\rceil}$ of optimal actions (and vice versa). Therefore, we can equivalently parameterise the reward function by the sequence $b_{1:s\lceil 1/\beta\rceil} := (b_1, \dots, b_{s\lceil 1/\beta\rceil}) \in [K]^{s\lceil 1/\beta\rceil}$. Using Lemma \ref{lem:regret_decomp}, we can re-write the regret $R_n(b_{1:s\lceil 1/\beta\rceil})$ as
\begin{equation*}
R_n(b_{1:s\lceil 1/\beta\rceil}) = \sum_{i=1}^{s\lceil 1/\beta\rceil}R_{m, i}(b_{1:i})\,,
\end{equation*}

where $R_{m, i}(b_{1:i}) = \mathbb{E}[\sum_{t=m(i-1)+1}^{mi}\frac{\Delta}{s}\mathbb{I}\{A_t \neq b_i\}]$ is the expected regret for the $i$\textsuperscript{th} sub-problem. Using this regret decomposition and Lemma \ref{lem:seq_lb} (and the fact that $n = s\lceil 1/\beta\rceil m$), we can lower bound the averaged expected regret as
\begin{equation*}
\frac{1}{K^{s\lceil 1/\beta\rceil}}\sum_{b_{1:s\lceil 1/\beta\rceil} \in [K]^{s\lceil 1/\beta\rceil}}R_n(b_{1:s\lceil 1/\beta\rceil}) \geq \frac{1}{8}s\lceil 1/\beta\rceil\sqrt{Km} = \frac{1}{8}\sqrt{\lceil 1/\beta\rceil Ksn} \geq \frac{1}{8}\sqrt{(1/\beta)Ksn}\,.
\end{equation*}

Lemma \ref{lem:seq_lb} requires us to set $\Delta/s = \sqrt{K/m}/2$. The only thing left to do is to find the values of $m$ such that $\Delta = s\sqrt{\frac{K}{4m}} \leq \exp(-s^{\beta}K^{\beta\lceil1/\beta\rceil}/2)$, which ensures that the exponential decay condition is satisfied. This constraint for $m$ can be rearranged into
\begin{equation*}
m \geq \tfrac{1}{4}s^2K\exp(s^{\beta}K^{\beta\lceil1/\beta\rceil})\,.
\end{equation*}

Thus, we can choose any $m$ satisfying $m \geq s^2K\exp(s^{\beta}K^{\beta\lceil1/\beta\rceil})$. Since the condition on $n$ was $n = sm$ for some integer $m \geq \lceil1/\beta\rceil s^2K\exp(s^{\beta}K^{\beta\lceil1/\beta\rceil})$, we can indeed factorise $n$ into $n = s\lceil 1/\beta\rceil m$ for some $m \geq s^2K\exp(s^{\beta}K^{\beta\lceil1/\beta\rceil})$.
\end{proof}

\subsection{Proof of Theorem \ref{thm:uncountable_lb}}
\label{sec:uncount_lb_proof}

For the convenience of the reader, we repeat the statement of Theorem \ref{thm:uncountable_lb} here.

\textit{Statement of Theorem \ref{thm:uncountable_lb}.} Consider the sparse nonparametric contextual bandit problem with uncountable sparsity described in Section \ref{sec:setting}. Let $\mathcal{A} = [K]$ for some $K \geq 2$ and let $n = sdm$ for some integer $m \geq s^{2 + 2/d}K^{3}$. Assume that the noise variables are standard Gaussian. For any policy, there exists a sequence of contexts $x_1, \dots, x_n \in \mathcal{X}$, parameters $w \in \mathbb{B}_1^s(1)$, $\theta_1, \dots, \theta_s \in \Theta \subset \mathbb{B}_2^d(1)$ and a uniformly Lipschitz continuous function $\phi$, with $\|\phi\|_{\infty} \leq 1$, such that
\begin{equation*}
R_n(f_{\nu}) \geq \frac{1}{8}\sqrt{Ksdn}.
\end{equation*}

\begin{proof}
The sequence of contexts $x_1, \dots, x_n$ is fixed in advance, and selected from the set $\{z_1, \dots, z_{sd}\}$. For each $i \in [sd]$ and $a \in [K]$, $z_{i,a}$ can be arbitrary, as long as each $z_{i,a}$ is distinct. In particular, we choose the sequence of contexts $x_1, \dots, x_n$ to be
\begin{equation*}
z_1, \dots, z_1, z_2, \dots, z_2, \dots, z_{sd}, \dots, z_{sd},
\end{equation*}

where each $z_i$ is repeated $m$ times. Next, we choose the parameters $w$ and $\theta_1, \dots, \theta_s$. First, we set $w_j = 1/s$ for every $j \in [s]$, which means that $\|w\|_1 = 1$. The parameters $\theta_1, \dots, \theta_s$ lie in a set $\Theta$, which we specify now. We choose $\Theta$ such that: a) $\Theta$ is a $\Delta$-packing of $\mathbb{B}_2^d(1)$ w.r.t.\ the $\ell_2$ norm, for some $\Delta \in (0, 1]$ to be chosen later; b) $\Theta$ is a union of $s$ disjoint sets $\Theta_1, \dots, \Theta_s$, each with cardinality $|\Theta_j| = K^{d}$. To guarantee that there is a set $\Theta$ that satisfies these requirements, we must choose $\Delta$ such that the $\Delta$-packing number $\mathcal{M}(\mathbb{B}_2^d(1), \|\cdot\|_2, \Delta)$ is at least $sK^{d}$. Using Lemma 5.5 and Lemma 5.7 in \citet{wainwright2019high}, $\mathcal{M}(\mathbb{B}_2^d(1), \|\cdot\|_2, \Delta)$ can be lower bounded as
\begin{equation*}
\mathcal{M}(\mathbb{B}_2^d(1), \|\cdot\|_2, \Delta) \geq \mathcal{N}(\mathbb{B}_2^d(1), \|\cdot\|_2, \Delta) \geq (\tfrac{1}{\Delta})^d,
\end{equation*}

which means that if $sK^{d} \leq (\frac{1}{\Delta})^d$, then we can construct a $\Delta$-packing of $\mathbb{B}_2^d(1)$ that has at least $sK^d$ elements. This inequality can be rearranged into $\Delta \leq 1/(s^{1/d}K)$. We will consider sequences of parameters $\theta_1, \dots, \theta_s$ where each $\theta_j$ is in $\Theta_j$. For each $j \in [s]$, since $\Theta_j$ has cardinality $K^{d}$, we can define a bijection $\omega^{j}:\Theta_j \to [K]^d$ between $\Theta_j$ and the set $[K]^d$ of sequences of actions of length $d$. We define the bijection $\rho:[sd] \to [s]\times[d]$ by
\begin{equation*}
\rho(i) = (\rho_1(i), \rho_2(i)) = (\lceil \tfrac{i}{d}\rceil, (i-1)~\mathrm{mod}~ d + 1).
\end{equation*}

Now, using $\rho$ and $\omega^1, \dots \omega^s$, we can define a bijection $b: \Theta_1 \times \cdots \times \Theta_s \to [K]^{sd}$ between $\Theta_1 \times \cdots \times \Theta_s$ and the set $[K]^{sd}$ of sequences of actions of length $d$. In particular, for any $(\theta_1, \dots \theta_s) \in \Theta_1 \times \cdots \times \Theta_s$ and $i \in [sd]$, we define
\begin{equation*}
b_i(\theta_1, \dots, \theta_s) = \omega_{\rho_2(i)}^{\rho_1(i)}(\theta_{\rho_1(i)}).
\end{equation*}

To define $\phi$, we introduce one more function. We define $\kappa: \Theta \to [s]$ to be the function that maps each $\theta \in \Theta$ to the unique integer $j \in [s]$ such that $\theta \in \Theta_j$. For each $i \in [sd]$, $a \in [K]$ and $\theta \in \Theta$, we define (with the same $\Delta \in (0, 1]$ as before)
\begin{equation*}
\phi(z_{i,a}, \theta) = \Delta\mathbb{I}\{\rho_1(i) = \kappa(\theta)\}\mathbb{I}\{a = \omega_{\rho_2(i)}^{\kappa(\theta)}(\theta)\}.
\end{equation*}

Since $\phi(z_{i,a}, \theta) \in \{0, \Delta\}$ and $\Theta$ is a $\Delta$-packing, we have
\begin{equation*}
\forall i \in [sd], a \in [K], \theta, \theta^{\prime} \in \Theta, ~|\phi(z_{i,a}, \theta) - \phi(z_{i,a}, \theta^{\prime})| \leq \Delta < \|\theta - \theta^{\prime}\|_2,
\end{equation*}

which means that $\phi$ satisfies the uniform Lipschitz continuity property. Using the definition of $\phi$, for any $\nu = (\theta_1, \dots, \theta_s) \in \Theta_1 \times \cdots \times \Theta_s$, the expected reward function is
\begin{equation*}
f_{\nu}(z_i, a) = \sum_{j=1}^{s}\tfrac{1}{s}\phi(z_{i,a}, \theta_j) = \tfrac{\Delta}{s}\mathbb{I}\{a = \omega_{\rho_2(i)}^{\rho_1(i)}(\theta_{\rho_1(i)})\} = \tfrac{\Delta}{s}\mathbb{I}\{a = b_i(\theta_1, \dots, \theta_s)\}.
\end{equation*}

We have another reward function that is in the same form as the reward function in (\ref{eqn:seq_reward}). Since $b$ is a bijection between $\Theta_1 \times \cdots \times \Theta_s$ and $[K]^{s\lceil 1/\beta\rceil}$, each $\nu \in \Theta_1 \times \cdots \times \Theta_s$ corresponds to a unique sequence $(b_1, \dots, b_{sd}) \in [K]^{sd}$ of optimal actions (and vice versa). Therefore, we can equivalently parameterise the reward function by the sequence $b_{1:sd} := (b_1, \dots, b_{sd}) \in [K]^{sd}$. Using Lemma \ref{lem:regret_decomp}, we can re-write the regret $R_n(b_{1:sd})$ as
\begin{equation*}
R_n(b_{1:sd}) = \sum_{i=1}^{sd}R_{m, i}(b_{1:i})\,,
\end{equation*}

where $R_{m, i}(b_{1:i}) = \mathbb{E}[\sum_{t=m(i-1)+1}^{mi}\frac{\Delta}{s}\mathbb{I}\{A_t \neq b_i\}]$ is the expected regret for the $i$\textsuperscript{th} sub-problem. Using this regret decomposition and Lemma \ref{lem:seq_lb} (and the fact that $n = sdm$), we can lower bound the averaged expected regret as
\begin{equation*}
\frac{1}{K^{sd}}\sum_{b_{1:sd} \in [K]^{sd}}R_n(b_{1:sd}) \geq \frac{1}{8}sd\sqrt{Km} = \frac{1}{8}\sqrt{Ksdn}\,.
\end{equation*}

Lemma \ref{lem:seq_lb} requires us to set $\Delta/s = \sqrt{K/m}/2$. The only thing left to do is to find the values of $m$ such that $\Delta = s\sqrt{\frac{K}{4m}} \leq 1/(s^{1/d}K)$, which ensures that $|\Theta|$ is smaller than $\mathcal{M}(\mathbb{B}_2^d(1), \|\cdot\|_2, \Delta)$. This constraint for $m$ can be rearranged into
\begin{equation*}
m \geq \tfrac{1}{4}s^{2 + 2/d}K^{3}\,.
\end{equation*}

Thus we can choose any $m$ satisfying $m \geq s^{2 + 2/d}K^{3}$.
\end{proof}

\section{Proof of Theorem \ref{thm:upperbounds_cs}}
\label{sec:cs_ub_proof}

To bound the expected regret of FGTS, we use the decoupling technique developed by \citet{zhang2022feel}. Theorem 1 in \citet{zhang2022feel} provides the following upper bound on the expected regret of FGTS.\medskip

\begin{theorem}[Theorem 1 in \citet{zhang2022feel}]
\label{thm:expregretbound}
Consider FGTS with the posterior defined in \eqref{eq:postnu} and the likelihood defined in \eqref{eq:fgtslikelihood}. Suppose that the prior $p_1$ is chosen such that $\mathbb{P}_{\nu \sim p_1}[\max_{x, a}|f_{\nu}(x, a)| \leq 1] = 1$. For any $\eta \leq 1/4$ and any $\lambda > 0$,
\begin{equation}
R_n(f^{*})\leq \frac{\lambda K n}{\eta}+6 \lambda n-\frac{1}{\lambda} Z_n\,,\;\;\text{where}\;\; Z_n:=\mathbb{E} \log \mathbb{E}_{\nu \sim p_1} \exp \bigg(-\sum_{t=1}^n\Delta L(\nu,X_t, A_t, Y_t)\bigg)\,,\label{eq:Zt}
\end{equation}
and
\begin{equation*}
\Delta L(\nu, x, a, y):=\eta\left[(f_{\nu}(x, a)-y)^2-\left(f^{*}(x, a)-y\right)^2\right]-\lambda\left[f_\nu(x)-f^{*}(x)\right]\,.
\end{equation*}
\end{theorem}

Note that $\Delta L(\nu, x, a, y)$ is the logarithm of the likelihood ratio with parameter $\nu^*$ in the denominator and $\nu$ in the numerator. Hence, we call $\Delta L$ the log-likelihood ratio. We use the shorthand $\Delta L_t(\nu) := \Delta L_t(\nu, X_t, A_t, Y_t)$. For countable sparsity, we use $\Delta L_t(\nu)$ and $\Delta L_t(w)$ interchangeably.

The proof of Theorem \ref{thm:upperbounds_cs} uses Theorem \ref{thm:expregretbound} and a bound on the log partition function $Z_n$ (defined in \eqref{eq:Zt}). To bound $Z_n$, we use some auxiliary lemmas. In Section \ref{sec:cs_ub_auxiliary}, we state and prove these auxiliary lemmas. In Section \ref{sec:cs_ub_Zn}, we then state and prove a bound on $Z_n$. Finally, in Section \ref{sec:cs_ub_proof_sub}, we prove Theorem \ref{thm:upperbounds_cs}.

We recall here some notation introduced in Sections \ref{defSNPB} and \ref{sec:ub_cs}. We use $s$ and $S$ to denote the sparsity and support of $w^*$. For any non-empty subset $M \subseteq \mathbb{N}$, we define $\mathcal{W}_M := \{w: \|w\|_1 \leq 1, ~w_i = 0 ~\forall i \notin M\}$. We let $\bar{w}$ denote the projection of the sequence $w^*$ onto the set of parameter sequences with support contained in $[\deff]$. Thus $\bar{w}$ is the sequence such that for all $i \in S \cap [\deff]$, $\bar{w}_i= w_i^*$, and for all $i \notin S \cap [\deff]$, $\bar{w}_i = 0$. We let $\bar{S} = \mathrm{supp}(\bar{w}) = \mathrm{supp}(w^*) \cap [\deff]$. Note that if $\bar{S} = \emptyset$, then the regret of any algorithm is $\mathcal{O}(\sqrt{n})$. In particular,
\begin{align*}
R_n(f^*) &= \mathbb{E}\left[{\textstyle \sum_{t=1}^{n}\max_{a \in [K]}}\left\{{\textstyle \sum_{i=\deff+1}^{\infty}}w_i^*\phi_i(X_{t,a})\right\} - {\textstyle \sum_{i=\deff+1}^{\infty}}w_i^*\phi_i(X_{t,A_t})\right]\\
&\leq 2n\|w^*\|_1\|\phi_{\deff + 1}\|_{\infty} \leq 2\sqrt{n}\,.
\end{align*}

Therefore, we continue under the assumption that $\bar{S} \neq \emptyset$. For each $c \in (0, 1]$, we define the set $\mathcal{W}_c = \{(1-c)\bar{w} + cw: w \in \mathcal{W}_{\bar{S}}\} \subseteq \mathcal{W}_{\bar{S}}$. We notice that for every $w \in \mathcal{W}_c$,
\begin{equation}
\|\bar{w} - w\|_1 = \|\bar{w} - (1-c)\bar{w} - cw^{\prime}\|_1 = c\|\bar{w} - w^{\prime}\|_1 \leq 2c\,.\label{eqn:count_norm_c_bound}
\end{equation}

\subsection{Auxiliary Lemmas}
\label{sec:cs_ub_auxiliary}

The first auxiliary lemma provides an alternative expression for the log-likelihood ratio $\Delta L$.\medskip

\begin{lemma}
\begin{align*}
\Delta L(\nu,X_t, A_t, Y_t) &= \eta (f_{\nu}(X_t, A_t)-f^*(X_t, A_t))^2 - 2\eta \epsilon_t(f_{\nu}(X_t, A_t)-f^*(X_t, A_t))\\
&+ \lambda (f_\nu(X_t) - f^*(X_t))\,.
\end{align*}
\label{lem:dl_alternative}
\end{lemma}

\begin{proof}
By definition of $\Delta L$, we have
\begin{align*}
\Delta L(\nu,X_t, A_t, Y_t) = \eta\left[(f_{\nu}(X_t, A_t)-Y_t)^2-\left(f^*(X_t, A_t)-Y_t\right)^2\right] -\lambda(f_\nu(X_t))-f^*(X_t))\,.
\end{align*}

By expanding and rearranging the squared terms, we obtain
\begin{align*}
(f_{\nu}(X_t, A_t)-Y_t)^2-\left(f^*(X_t, A_t)-Y_t\right)^2 &= (f_{\nu}(X_t, A_t)-f^*(X_t, A_t)-\epsilon_t)^2 - \epsilon_t^2\\
&= (f_{\nu}(X_t, A_t)-f^*(X_t, A_t))^2\\
&- 2\epsilon_t(f_{\nu}(X_t, A_t)-f^*(X_t, A_t)).
\end{align*}
\end{proof}

The next lemma utilises the uniform decay condition. It shows that for any parameter sequence $w$ with support contained in $[\deff]$, the difference between $f_w$ and $f^*$ can be bounded in terms of the $\ell_1$ distance between $w$ and $\bar{w}$ and a small approximation error caused by ignoring any components of $w^*$ with indices greater than $\deff$.\medskip

\begin{lemma}
For any $w$ with support contained in $[\deff]$, and all $x \in \mathcal{X}$ and $a \in [K]$,
\begin{equation*}
|f_w(x, a) - f^*(x, a)| \leq \|w - \bar{w}\|_1 + 1/\sqrt{n}\,.
\end{equation*}
\label{lem:trunc_err}
\end{lemma}

\begin{proof}
Using the triangle inequality, we obtain
\begin{align*}
|f_w(x, a) - f^*(x, a)| &= \left|\sum_{i=1}^{\deff}(w_i - \bar{w}_i)\phi_i(x_a) - \sum_{i=\deff + 1}^{\infty}w_i^*\phi_i(x_a)\right|\\
&\leq \|w - \bar{w}\|_1\|\phi_1\|_{\infty} + \|w^*\|_1\|\phi_{\deff + 1}\|_{\infty}\\
&\leq \|w - \bar{w}\|_1 + 1/\sqrt{n}\,.
\end{align*}

The fact that $\|\phi_{\deff+1}\|_{\infty} \leq 1/\sqrt{n}$ follows from the definition of $\deff$.
\end{proof}

As a consequence of this lemma, we also have
\begin{equation*}
(f_{w}(x, a) - f^*(x, a))^2 \leq \|w - \bar{w}\|_1^2 + 2\|w - \bar{w}\|_1/\sqrt{n} + 1/n\,.
\end{equation*}

In addition, since $|f_{w}(x) - f^*(x)| \leq \max_{a \in [K]}|f_{w}(x, a) - f^*(x, a)|$, we also have
\begin{equation*}
|f_{w}(x) - f^*(x)| \leq \|w - \bar{w}\|_1 + 1/\sqrt{n}\,.
\end{equation*}

Using Lemma \ref{lem:trunc_err}, we obtain the following exponential moment bound, which we will use later to control the terms depending on $\epsilon_1, \dots, \epsilon_n$ that appear in Lemma \ref{lem:dl_alternative}.\medskip

\begin{lemma}
For any fixed $w \in \mathcal{W}_c$,
\begin{equation*}
\mathbb{E}\left[\exp({\textstyle \sum_{t=1}^{n}}-2\eta(f_{w}(X_t, A_t) - f^*(X_t, A_t))\epsilon_t)\right] \leq \exp\left(\tfrac{\eta^2(4c^2n + 4c\sqrt{n} + 1)}{2}\right)\,.
\end{equation*}
\label{lem:exp_moment_w}
\end{lemma}

\begin{proof}
We recall that $\mathbb{E}_t[\cdot] = \mathbb{E}[\cdot\mid \mathcal{F}_{t-1}, X_t, A_t]$ and that each $\epsilon_t$ is sub-Gaussian, meaning $\mathbb{E}_t[\exp(\lambda \epsilon_t)] \leq \exp(\lambda^2/8)$ for any $\lambda\in \mathbb{R}$. Using this, Lemma \ref{lem:trunc_err} and the fact that $w \in \mathcal{W}_c$, we have
\begin{align*}
&\mathbb{E}\left[\exp({\textstyle \sum_{t=1}^{n}}-2\eta(f_{w}(X_t, A_t) - f^*(X_t, A_t))\epsilon_t)\right]\\
&= \mathbb{E}\left[\exp({\textstyle \sum_{t=1}^{n-1}}-2\eta(f_{w}(X_t, A_t) - f^*(X_t, A_t))\epsilon_t)\mathbb{E}_n\left[\exp(-2\eta(f_{w}(X_n, A_n) - f^*(X_n, A_n))\epsilon_n)\right]\right]\\
&\leq \mathbb{E}\left[\exp({\textstyle \sum_{t=1}^{n-1}}-2\eta(f_{w}(X_t, A_t) - f^*(X_t, A_t))\epsilon_t)\exp(\tfrac{\eta^2(f_{w}(X_n, A_n) - f^*(X_n, A_n))^2}{2})\right]\\
&\leq \mathbb{E}\left[\exp({\textstyle \sum_{t=1}^{n-1}}-2\eta(f_{w}(X_t, A_t) - f^*(X_t, A_t))\epsilon_t)\right]\exp\left(\tfrac{\eta^2(4c^2 + 4c/\sqrt{n} + 1/n)}{2}\right)\,,
\end{align*}
where the last inequality follows from Lemma \ref{lem:trunc_err} and \eqref{eqn:count_norm_c_bound}. 
By iterating this argument, we obtain the inequality in the statement of the Lemma.
\end{proof}

The next lemma provides a bound on the covering number of $\mathcal{W}_c$.\medskip

\begin{lemma}[Lemma 7 in \citet{wainwright2019high}]
For any $p \geq 1$, $d \geq 1$, $c > 0$ and $\Delta > 0$,
\begin{equation*}
\mathcal{N}(\mathbb{B}_p^d(c), \|\cdot\|_p, \Delta) \leq (1 + \tfrac{2c}{\Delta})^d\,. 
\end{equation*}
\label{lem:l1_cover}
\end{lemma}

It is easy to see that there is a surjective isometric embedding from the set $\mathcal{W}_c$ (with the $\ell_1$ norm) to the ball $\mathbb{B}_1^{\bar{s}}(c)$ (also with the $\ell_1$ norm). In particular, to embed $w \in \mathcal{W}_c$ into $\mathbb{B}_1^{\bar{s}}(c)$, one can subtract $(1-c)\bar{w}$ and then remove all the components corresponding to indices not in $\bar{S}$. Therefore, a consequence of Lemma \ref{lem:l1_cover} is that $\mathcal{N}(\mathcal{W}_c, \|\cdot\|_1, \Delta) \leq (1 + \tfrac{2c}{\Delta})^{\bar{s}}$. The final auxiliary lemma controls the expected value of the maximum of the noise process originating from Lemma \ref{lem:dl_alternative}.\medskip

\begin{lemma}
For any $c \in (0, 1]$ and $\Delta > 0$,
\begin{equation*}
\mathbb{E}\left[\max_{w \in \mathcal{W}_c}\left\{{\textstyle \sum_{t=1}^{n}}-2\eta(f_{w}(X_t, A_t) - f^*(X_t, A_t))\epsilon_t\right\}\right] \leq \bar{s}\log(1 + \tfrac{2c}{\Delta}) + \tfrac{\eta^2(4c^2n + 4c\sqrt{n} + 1)}{2} + \Delta\eta n\,.
\end{equation*}
\label{lem:cs_gaussian_comp}
\end{lemma}

\begin{proof}
We set $\mathcal{W}_{c, \Delta}$ to be any minimal $\ell_1$-norm $\Delta$-covering of $\mathcal{W}_c$. We define $[w] := \argmin_{w^{\prime} \in \mathcal{W}_{c, \Delta}}\|w - w^{\prime}\|_1$ to be the $\ell_1$-norm projection of $w \in \mathcal{W}_c$ into $\mathcal{W}_{c, \Delta}$. The first step is to replace the maximum over the infinite set $\mathcal{W}_c$ by a maximum over the finite set $\mathcal{W}_{c, \Delta}$ and a discretisation error. We have
\begin{align*}
\max_{w \in \mathcal{W}_c}\left\{{\textstyle \sum_{t=1}^{n}}-2\eta(f_{w}(X_t, A_t) - f^*(X_t, A_t))\epsilon_t\right\} &\leq \max_{w \in \mathcal{W}_c}\left\{{\textstyle \sum_{t=1}^{n}}-2\eta(f_{[w]}(X_t, A_t) - f^*(X_t, A_t))\epsilon_t\right\}\\
&+ \max_{w \in \mathcal{W}_c}\left\{{\textstyle \sum_{t=1}^{n}}-2\eta(f_{w}(X_t, A_t) - f_{[w]}(X_t, A_t))\epsilon_t\right\}\\
&= \max_{w \in \mathcal{W}_{c, \Delta}}\left\{{\textstyle \sum_{t=1}^{n}}-2\eta(f_{w}(X_t, A_t) - f^*(X_t, A_t))\epsilon_t\right\}\\
&+ \max_{w \in \mathcal{W}_{c}}\left\{{\textstyle \sum_{t=1}^{n}}-2\eta(f_{w}(X_t, A_t) - f_{[w]}(X_t, A_t))\epsilon_t\right\}\,.
\end{align*}

To bound the expectation of the maximum over $\mathcal{W}_{c, \Delta}$, we use Jensen's inequality, Lemma \ref{lem:exp_moment_w} and Lemma \ref{lem:l1_cover} to obtain
\begin{align*}
\mathbb{E}&\left[\max_{w \in \mathcal{W}_{c, \Delta}}\big\{{\textstyle \sum_{t=1}^{n}}-2\eta(f_{w}(X_t, A_t) - f^*(X_t, A_t))\epsilon_t\big\}\right]\\
&\leq \log\mathbb{E}\left[\max_{w \in \mathcal{W}_{c, \Delta}}\big\{\exp\big({\textstyle \sum_{t=1}^{n}}-2\eta(f_{w}(X_t, A_t) - f^*(X_t, A_t))\epsilon_t\big)\big\}\right]\\
&\leq \log\mathbb{E}\left[{\textstyle \sum_{w \in \mathcal{W}_{c, \Delta}}}\big\{\exp\big({\textstyle \sum_{t=1}^{n}}-2\eta(f_{w}(X_t, A_t) - f^*(X_t, A_t))\epsilon_t\big)\big\}\right]\\
&\leq \log\left(|\mathcal{W}_{c, \Delta}|\exp\left(\tfrac{\eta^2(4c^2n + 4c\sqrt{n} + 1)}{2}\right)\right)\\
&= \log(\mathcal{N}(\mathcal{W}_c, \|\cdot\|_1, \Delta)) + \tfrac{\eta^2(4c^2n + 4c\sqrt{n} + 1)}{2}\\
&\leq \bar{s}\log(1 + \tfrac{2c}{\Delta}) + \tfrac{\eta^2(4c^2n + 4c\sqrt{n} + 1)}{2}\,.
\end{align*}

Since, $\mathcal{W}_{c, \Delta}$ is a $\Delta$-covering and $\|\phi_i\|_{\infty} \leq 1$, we have $|f_w(x, a) - f_{[w]}(x, a)| \leq \Delta$ for all $x \in \mathcal{X}$ and $a \in [K]$. Also, since each $\epsilon_t$ is 1/2-sub-Gaussian, $\mathbb{E}[\epsilon_t^2] \leq 1/4$. Using these facts and the Cauchy-Schwartz inequality, we can bound the discretisation error as
\begin{align*}
\mathbb{E}&\left[\max_{w \in \mathcal{W}_{c}}\left\{{\textstyle \sum_{t=1}^{n}}-2\eta(f_{w}(X_t, A_t) - f_{[w]}(X_t, A_t))\epsilon_t\right\}\right]\\
&\leq \mathbb{E}\left[\max_{w \in \mathcal{W}_{c}}\left\{\sqrt{{\textstyle \sum_{t=1}^{n}}4\eta^2(f_w(X_t, A_t) - f_{[w]}(X_t, A_t))^2}\right\}\sqrt{{\textstyle \sum_{t=1}^{n}}\epsilon_t^2}\right]\\
&\leq 2\eta\Delta\sqrt{n}\mathbb{E}\left[\sqrt{{\textstyle \sum_{t=1}^{n}}\epsilon_t^2}\right]\\
&\leq 2\eta\Delta\sqrt{n}\sqrt{{\textstyle \sum_{t=1}^{n}}\mathbb{E}[\epsilon_t^2]}\\
&\leq \eta\Delta n\,.
\end{align*}
\end{proof}

\subsection{Bounding the Log Partition Function}
\label{sec:cs_ub_Zn}

Using the auxiliary lemmas established in the previous subsection, we can now prove a bound on the log partition function $Z_n$.\medskip

\begin{lemma}
If we use the prior $p_1$ in \eqref{eqn:cs_prior}, then for every $n$,
\begin{equation*}
-Z_n \leq \|w^*\|_0\log(8e\deff n) + 2\eta^2 + 5\eta + 2\lambda\sqrt{n}\,.
\end{equation*}
\label{lem:z_cs}
\end{lemma}

\begin{proof}
For each $c \in (0, 1]$, we define the set $\mathcal{W}_c = \{(1-c)\bar{w} + cw: w \in \mathcal{W}_{\bar{S}}\} \subseteq \mathcal{W}_{\bar{S}}$ and define the event
\begin{equation*}
\mathcal{E}_c := \{w \in \mathcal{W}_c\}\,.
\end{equation*}

When $w \sim p_1(w\mid M=\bar{S}) = \mathcal{U}(\mathcal{W}_{\bar{S}})$, we have
\begin{align*}
\mathbb{P}_{w \sim p_1(\cdot\mid M=\bar{S})}[\mathcal{E}_c] = \tfrac{\mathrm{Vol}(\mathcal{W}_c)}{\mathrm{Vol}(\mathcal{W}_{\bar{S}})} = c^{\bar{s}}\,.
\end{align*}

Since we know that $p_1(M) = 2^{-|M|}\binom{\deff}{|M|}^{-1}(\sum_{i=1}^{\deff}2^{-i})$, we have
\begin{align*}
Z_n &= \mathbb{E}\left[\log\mathbb{E}_{w \sim p_1}\left[\exp(-{\textstyle \sum_{t=1}^{n}}\Delta L(w, X_t, A_t, Y_t))\right]\right]\\
&\geq \mathbb{E}\left[\log p_1(\bar{S})\mathbb{E}_{w \sim p_1|\bar{S}}\left[\exp({\textstyle \sum_{t=1}^{n}}\Delta L(w, X_t, A_t, Y_t))\right]\right]\\
&\geq \mathbb{E}\left[\log p_1(\bar{S})\mathbb{P}_{w \sim p_1|M=\bar{S}}[\mathcal{E}_c]\min_{\nu \in \mathcal{W}_c}\left\{\exp(-{\textstyle \sum_{t=1}^{n}}\Delta L(w, X_t, A_t, Y_t))\right\}\right]\\
&= \log(p_1(\bar{S})) + \log(\mathbb{P}_{w \sim p_1|M=\bar{S}}[\mathcal{E}_c]) - \mathbb{E}\left[\max_{w \in \mathcal{W}_c}\left\{{\textstyle \sum_{t=1}^{n}}\Delta L(w, X_t, A_t, Y_t)\right\}\right]\\
&= -\bar{s}\log(2) - \log({\textstyle\binom{\deff}{\bar{s}}}) - \log({\textstyle \sum_{i=1}^{\deff}}2^{-i}) - \bar{s}\log(1/c) - \mathbb{E}\left[\max_{w \in \mathcal{W}_c}\left\{{\textstyle \sum_{t=1}^{n}}\Delta L(w, X_t, A_t, Y_t)\right\}\right]\,.
\end{align*}

Next, using Lemma \ref{lem:dl_alternative}, Lemma \ref{lem:trunc_err} and \eqref{eqn:count_norm_c_bound}, for any $w \in \mathcal{W}_c$, we have
\begin{align*}
{\textstyle \sum_{t=1}^{n}}\Delta L(w, X_t, A_t, Y_t) &= {\textstyle \sum_{t=1}^{n}}\eta (f_{w}(X_t, A_t)-f^*(X_t, A_t))^2 + \lambda(f_w(X_t) - f^*(X_t))\\
&+ {\textstyle \sum_{t=1}^{n}}-2\eta(f_{w}(X_t, A_t) - f^*(X_t, A_t))\epsilon_t\\
&\leq 4c^2\eta n + 4c\eta \sqrt{n} + \eta + 2c\lambda n + \lambda \sqrt{n}\\
&+ {\textstyle \sum_{t=1}^{n}}-2\eta(f_{w}(X_t, A_t) - f^*(X_t, A_t))\epsilon_t\,.
\end{align*}

Therefore,
\begin{align*}
\mathbb{E}\left[\max_{w \in \mathcal{W}_c}\left\{{\textstyle \sum_{t=1}^{n}}\Delta L(w, X_t, A_t, Y_t)\right\}\right] &\leq 4c^2\eta n + 4c\eta \sqrt{n} + \eta + 2c\lambda n + \lambda \sqrt{n}\\
&+ \mathbb{E}\left[\max_{w \in \mathcal{W}_c}\left\{{\textstyle \sum_{t=1}^{n}}-2\eta(f_{w}(X_t, A_t) - f^*(X_t, A_t))\epsilon_t\right\}\right]\,.
\end{align*}

Using Lemma \ref{lem:cs_gaussian_comp}, for any $c \in (0, 1]$ and $\Delta > 0$, we also have
\begin{equation*}
\mathbb{E}\left[\max_{w \in \mathcal{W}_c}\left\{{\textstyle \sum_{t=1}^{n}}-2\eta(f_{w}(X_t, A_t) - f^*(X_t, A_t))\epsilon_t\right\}\right] \leq \bar{s}\log(1 + \tfrac{2c}{\Delta}) + \tfrac{\eta^2(4c^2n + 4c\sqrt{n} + 1)}{2} + \Delta\eta n\,.
\end{equation*}

If we choose $c = 1/(2\sqrt{n})$ and $\Delta = 1/n$, then this bound becomes
\begin{equation*}
\mathbb{E}\left[\max_{w \in \mathcal{W}_c}\left\{{\textstyle \sum_{t=1}^{n}}-2\eta(f_{w}(X_t, A_t) - f^*(X_t, A_t))\epsilon_t\right\}\right] \leq \bar{s}\log(1 + \sqrt{n}) + 2\eta^2 + \eta\,.
\end{equation*}

Thus, with the choice $c = 1/(2\sqrt{n})$, we obtain the bound
\begin{equation*}
\mathbb{E}\left[\max_{w \in \mathcal{W}_c}\left\{{\textstyle \sum_{t=1}^{n}}\Delta L(w, X_t, A_t, Y_t)\right\}\right] \leq \bar{s}\log(1 + \sqrt{n}) + 2\eta^2 + 5\eta + 2\lambda\sqrt{n}\,.
\end{equation*}

If we combine everything, and use the inequality $\binom{\deff}{\bar{s}} \leq (e\deff/\bar{s})^{\bar{s}}$, then we obtain
\begin{align*}
-Z_n &\leq \bar{s}\log(2) + \log({\textstyle\binom{\deff}{\bar{s}}}) + \log({\textstyle \sum_{i=1}^{\deff}}2^{-i}) + \bar{s}\log(2\sqrt{n}) + \bar{s}\log(1 + \sqrt{n}) + 2\eta^2 + 5\eta + 2\lambda\sqrt{n}\\
&\leq \|w^*\|_0\big(\log(2) + \log(e\deff) + 2\log(2\sqrt{n})\big) + 2\eta^2 + 5\eta + 2\lambda\sqrt{n}\\
&\leq \|w^*\|_0\log(8e\deff n) + 2\eta^2 + 5\eta + 2\lambda\sqrt{n}\,.
\end{align*}
\end{proof}

\subsection{Proof of Theorem \ref{thm:upperbounds_cs}}
\label{sec:cs_ub_proof_sub}

\begin{proof}
Using Theorem \ref{thm:expregretbound}, and then Lemma \ref{lem:z_cs} (with $\eta = 1/4$), we have
\begin{align*}
R_n(f^*) &\leq \lambda(4K + 6)n - \frac{1}{\lambda}Z_n\\
&\leq \lambda(4K + 6)n + \frac{1}{\lambda}\big(\|w^*\|_0\log(8e\deff n) + 2\big) + 2\sqrt{n}\,.
\end{align*}

If we choose
\begin{equation*}
\lambda = \sqrt{\frac{\log(8e\deff n)}{(4K + 6)n}}\,,
\end{equation*}

then we obtain the regret bound
\begin{align*}
R_n(f^*) &\leq (\|w^*\|_0 + 1)\sqrt{(4K + 6)n\log(8e\deff n)} + 2\sqrt{\frac{(4K + 6)n}{\log(8e\deff n)}} + 2\sqrt{n}\\
&= \mathcal{O}\Big(\|w^*\|_0\sqrt{Kn\log(\deff n)}\Big)\,.
\end{align*}

If $s$ is a known upper bound on $\|w^*\|_0$ and we choose
\begin{equation*}
\lambda = \sqrt{\frac{s\log(8e\deff n) + 2}{(4K + 6)n}}\,,
\end{equation*}

then we obtain the regret bound
\begin{align*}
R_n(f^*) &\leq 2\sqrt{(4K + 6)(s\log(8e\deff n) + 2)n} + 2\sqrt{n}\\
&= \mathcal{O}(\sqrt{Ksn\log(\deff n)})\,.
\end{align*}
\end{proof}

\section{Proof of Theorem \ref{thm:upperbounds_us}}
\label{sec:us_ub_proof}

This section follows a similar structure to the previous one. In Section \ref{sec:us_ub_auxiliary}, we state and prove some auxiliary lemmas. In Section \ref{sec:us_ub_Zn}, we state and prove a bound on $Z_n$. Finally, in Section \ref{sec:us_ub_proof_sub}, we prove Theorem \ref{thm:upperbounds_us}.

In this section, unless stated otherwise, we let $s = \|w^*\|_0$ denote the true sparsity. We recall here some notation introduced in Section \ref{sec:ub_us}. For each $c \in (0,1]$ we define the sets
\begin{equation*}
\mathcal{W}_c := \{(1-c)w^* + cw: w \in \mathbb{B}_1^{s}(1)\}, \qquad \Theta_{i,c} := \{(1-c)\theta_i^* + c\theta: \theta \in \mathbb{B}_2^d(1)\}\,.
\end{equation*}

We notice that for each $w \in \mathcal{W}_c$,
\begin{equation*}
\|w^* - w\|_1 = c\|w^* - w^{\prime}\|_1 \leq c\|w^*\|_1 + c\|w^{\prime}\|_1 \leq 2c\,,
\end{equation*}

where $w^{\prime}$ is some element in $\mathbb{B}_1^{s}(1)$. Similarly, for each $i \in [s]$ and $\theta_i \in \Theta_{i,c}$, we have $\|\theta_i^* - \theta_i\|_2 \leq 2c$. For each $c \in (0, 1]$, let $N_c = \mathcal{W}_c \times \Theta_{1,c} \times \cdots \times \Theta_{s,c}$

\subsection{Auxiliary Lemmas}
\label{sec:us_ub_auxiliary}

First, we show that if $\phi$ satisfies the Lipschitz property in Definition \ref{def:uni_lip}, then for any $x \in \mathcal{X}$ and $a \in [K]$, the function value $f_{\nu}(x, a)$ changes smoothly as $\nu$ is varied.\medskip

\begin{lemma}
For any $m \in \mathbb{N}$ and any $\nu, \nu^{\prime} \in \mathbb{R}^{m} \times \Theta^{m}$,
\begin{equation*}
\forall x \in \mathcal{X}, a \in [K], ~|f_{\nu}(x, a) - f_{\nu^{\prime}}(x, a)| \leq \max_{i \in [m]}\{\|\theta_i - \theta_i^{\prime}\|_2\} + \|w - w^{\prime}\|_1.
\end{equation*}
\label{lem:lip}
\end{lemma}

\begin{proof}
Using the triangle inequality and the uniform Lipschitz property, we have
\begin{align*}
|f_{\nu}(x, a) - f_{\nu^{\prime}}(x, a)| &= \left|\sum_{i=1}^{m}(w_i\phi(x_a, \theta_i) - w_i^{\prime}\phi(x_a, \theta_i^{\prime}))\right|\\
&= \left|\sum_{i=1}^{m}(w_i\phi(x_a, \theta_i) - w_i\phi(x_a, \theta_i^{\prime}) + w_i\phi(x_a, \theta_i^{\prime}) - w_i^{\prime}\phi(x_a, \theta_i^{\prime}))\right|\\
&\leq \sum_{i=1}^{m}|w_i||\phi(x_a, \theta_i) -\phi(x_a, \theta_i^{\prime})| + \sum_{i=1}^{m}|w_i - w_i^{\prime}||\phi(x_a, \theta_i^{\prime})|\\
&\leq \max_{i \in [m]}\{\|\theta_i - \theta_i^{\prime}\|_2\} + \|w - w^{\prime}\|_1\,.
\end{align*}
\end{proof}

Note that this upper bound also applies to $f_{\nu}(x)$ and $f_{\nu^{\prime}}(x)$. In particular,
\begin{align*}
|f_{\nu}(x) - f_{\nu^{\prime}}(x)| &= |\max_{a \in [K]}\{f_{\nu}(x, a)\} - \max_{a \in [K]}\{f_{\nu^{\prime}}(x, a)\}|\\
&\leq |\max_{a \in [K]}\{f_{\nu}(x, a) - f_{\nu^{\prime}}(x, a)\}|\\
&\leq \max_{i \in [m]}\{\|\theta_i - \theta_i^{\prime}\|_2\} + \|w - w^{\prime}\|_1\,.
\end{align*}

Using Lemma \ref{lem:lip}, we obtain the following exponential moment bound.\medskip

\begin{lemma}
For any fixed $\nu \in N_c$,
\begin{equation*}
\mathbb{E}\left[\exp\big({\textstyle \sum_{t=1}^{n}}-2\eta(f_{\nu}(X_t, A_t) - f^*(X_t, A_t)\epsilon_t\big)\right] \leq \exp\big(8c^2\eta^2n\big)\,.
\end{equation*}
\label{lem:exp_moment_nu}
\end{lemma}

\begin{proof}
Using the sub-Gaussian property of $\epsilon_1, \dots, \epsilon_n$, Lemma \ref{lem:lip} and the fact that $\nu \in N_c$, we obtain
\begin{align*}
&\mathbb{E}\left[\exp\big({\textstyle \sum_{t=1}^{n}}-2\eta(f_{\nu}(X_t, A_t) - f^*(X_t, A_t))\epsilon_t\big)\right]\\
&= \mathbb{E}\left[\exp\big({\textstyle \sum_{t=1}^{n-1}}-2\eta(f_{\nu}(X_t, A_t) - f^*(X_t, A_t))\epsilon_t\big)\mathbb{E}_n\left[\exp\big(-2\eta(f_{\nu}(X_n, A_n) - f^*(X_n, A_n))\epsilon_n\big)\right]\right]\\
&\leq \mathbb{E}\left[\exp\big({\textstyle \sum_{t=1}^{n-1}}-2\eta(f_{\nu}(X_t, A_t) - f^*(X_t, A_t))\epsilon_t\big)\exp\big(\tfrac{\eta^2(f_{\nu}(X_n, A_n) - f^*(X_n, A_n))^2}{2}\big)\right]\\
&\leq \mathbb{E}\left[\exp\big({\textstyle \sum_{t=1}^{n-1}}-2\eta(f_{\nu}(X_t, A_t) - f^*(X_t, A_t))\epsilon_t\big)\right]\exp\big(8c^2\eta^2\big)\,.
\end{align*}

By iterating this argument, we obtain the inequality in the statement of the lemma.
\end{proof}

We set $\mathcal{W}_{c, \Delta}$ to be any minimal $\ell_1$-norm $\Delta$-covering of $\mathcal{W}_c$ and (for each $i \in s$) $\Theta_{i,c,\Delta}$ to be any minimal $\ell_2$-norm $\Delta$-covering of $\Theta_{i, c}$. We let $N_{c, \Delta} = \mathcal{W}_{c,\Delta} \times \Theta_{1,c,\Delta} \times \cdots \times \Theta_{s,c,\Delta}$. Using Lemma \ref{lem:l1_cover}, we obtain a bound on the cardinality of $N_{c, \Delta}$.\medskip

\begin{corollary}
For any $c \in (0, 1]$ and $\Delta > 0$,
\begin{equation*}
|N_{c,\Delta}| \leq \big(1 + \tfrac{2c}{\Delta}\big)^{s(d+1)}\,.
\end{equation*}
\label{cor:nu_cover}
\end{corollary}

\begin{proof}
Using Lemma \ref{lem:l1_cover}, we have
\begin{equation*}
|\mathcal{W}_{c,\Delta}| \leq \mathcal{N}(\mathbb{B}_1^s(c), \|\cdot\|_1, \Delta) \leq \big(1 + \tfrac{2c}{\Delta}\big)^s\,,
\end{equation*}

and, for any $i \in [s]$,
\begin{equation*}
|\Theta_{i,c,\Delta}| \leq \mathcal{N}(\mathbb{B}_2^d(c), \|\cdot\|_2, \Delta) \leq \big(1 + \tfrac{2c}{\Delta}\big)^d\,.
\end{equation*}

Therefore,
\begin{equation*}
|N_{c,\Delta}| = |\mathcal{W}_{c,\Delta}| \times |\Theta_{1,c,\Delta}| \times \cdots \times |\Theta_{s,c,\Delta}| \leq \big(1 + \tfrac{2c}{\Delta}\big)^{s(d+1)}\,.
\end{equation*}
\end{proof}

The final auxiliary lemma is analogous to Lemma \ref{lem:cs_gaussian_comp}. It controls the expectation of the maximum of the noise process in from Lemma \ref{lem:dl_alternative} for the case of uncountable sparsity.\medskip

\begin{lemma}
For any $c \in (0, 1]$ and $\Delta > 0$,
\begin{equation*}
\mathbb{E}\left[\max_{\nu \in N_c}\left\{{\textstyle \sum_{t=1}^{n}}-2\eta(f_{\nu}(X_t, A_t) - f^*(X_t, A_t))\epsilon_t\right\}\right] \leq s(d+1)\log(1 + \tfrac{2c}{\Delta}) + 8c^2\eta^2n + 2\eta\Delta n\,.
\end{equation*}
\label{lem:us_gaussian_comp}
\end{lemma}

\begin{proof}
We define $[w] := \argmin_{w^{\prime} \in \mathcal{W}_{c, \Delta}}\|w - w^{\prime}\|_1$ to be the $\ell_1$-norm projection of $w \in \mathcal{W}_c$ into $\mathcal{W}_{c, \Delta}$. For each $i \in [s]$, we define $[\theta_i] := \argmin_{\theta^{\prime} \in \Theta_{i, c, \Delta}}\|\theta_i - \theta^{\prime}\|_2$ to be the $\ell_2$-norm projection of $\theta_i \in \Theta_{i,c}$ into $\Theta_{i, c, \Delta}$. For any $\nu \in N_c$, we define $[\nu] := ([w], [\theta_1], \dots, [\theta_s])$. Using Lemma \ref{lem:lip}, and the fact that $\mathcal{W}_{c,\Delta}$, $\Theta_{1,c,\Delta}, \dots, \Theta_{s, c, \Delta}$ are $\Delta$-coverings, we have
\begin{equation*}
\forall x \in \mathcal{X}, a \in [K], \quad |f_{\nu}(x, a) - f_{[\nu]}(x, a)| \leq \max_{i \in [s]}\{\|\theta_i - [\theta_i]\|_2\} + \|w - [w]\|_1 \leq 2\Delta\,.
\end{equation*}

The first step is to replace the maximum over the infinite set $N_c$ by a maximum over the finite set $N_{c, \Delta}$ and a discretisation error. We have
\begin{align*}
\max_{\nu \in N_c}\left\{{\textstyle \sum_{t=1}^{n}}-2\eta(f_{\nu}(X_t, A_t) - f^*(X_t, A_t))\epsilon_t\right\} &\leq \max_{\nu \in N_c}\left\{{\textstyle \sum_{t=1}^{n}}-2\eta(f_{[\nu]}(X_t, A_t) - f^*(X_t, A_t))\epsilon_t\right\}\\
&+ \max_{\nu \in N_c}\left\{{\textstyle \sum_{t=1}^{n}}-2\eta(f_{\nu}(X_t, A_t) - f_{[\nu]}(X_t, A_t))\epsilon_t\right\}\\
&= \max_{\nu \in N_{c, \Delta}}\left\{{\textstyle \sum_{t=1}^{n}}-2\eta(f_{\nu}(X_t, A_t) - f^*(X_t, A_t))\epsilon_t\right\}\\
&+ \max_{\nu \in N_{c}}\left\{{\textstyle \sum_{t=1}^{n}}-2\eta(f_{\nu}(X_t, A_t) - f_{[\nu]}(X_t, A_t))\epsilon_t\right\}\,.
\end{align*}

To bound the expectation of the maximum over $N_{c, \Delta}$, we use Jensen's inequality, Lemma \ref{lem:exp_moment_nu} and Corollary \ref{cor:nu_cover} to obtain
\begin{align*}
\mathbb{E}&\left[\max_{\nu \in N_{c, \Delta}}\big\{{\textstyle \sum_{t=1}^{n}}-2\eta(f_{\nu}(X_t, A_t) - f^*(X_t, A_t))\epsilon_t\big\}\right]\\
&\leq \log\mathbb{E}\left[\max_{\nu \in N_{c, \Delta}}\big\{\exp\big({\textstyle \sum_{t=1}^{n}}-2\eta(f_{\nu}(X_t, A_t) - f^*(X_t, A_t))\epsilon_t\big)\big\}\right]\\
&\leq \log\mathbb{E}\left[{\textstyle \sum_{\nu \in N_{c, \Delta}}}\big\{\exp\big({\textstyle \sum_{t=1}^{n}}-2\eta(f_{\nu}(X_t, A_t) - f^*(X_t, A_t))\epsilon_t\big)\big\}\right]\\
&\leq \log\left(|N_{c, \Delta}|\exp\left(8c^2\eta^2n\right)\right)\\
&= s(d+1)\log(1 + \tfrac{2c}{\Delta}) + 8c^2\eta^2n\,.
\end{align*}

Using the sub-Gaussian property of $\epsilon_1, \dots, \epsilon_n$ and the Cauchy-Schwartz inequality, we can bound the discretisation error as
\begin{align*}
\mathbb{E}&\left[\max_{\nu \in N_{c}}\left\{{\textstyle \sum_{t=1}^{n}}-2\eta(f_{\nu}(X_t, A_t) - f_{[\nu]}(X_t, A_t))\epsilon_t\right\}\right]\\
&\leq \mathbb{E}\left[\max_{\nu \in N_{c}}\left\{\sqrt{{\textstyle \sum_{t=1}^{n}}4\eta^2(f_{\nu}(X_t, A_t) - f_{[\nu]}(X_t, A_t))^2}\right\}\sqrt{{\textstyle \sum_{t=1}^{n}}\epsilon_t^2}\right]\\
&\leq 4\eta\Delta\sqrt{n}\mathbb{E}\left[\sqrt{{\textstyle \sum_{t=1}^{n}}\epsilon_t^2}\right]\\
&\leq 4\eta\Delta\sqrt{n}\sqrt{{\textstyle \sum_{t=1}^{n}}\mathbb{E}[\epsilon_t^2]}\\
&\leq 2\eta\Delta n\,.
\end{align*}
\end{proof}

\subsection{Bounding the Log Partition Function}
\label{sec:us_ub_Zn}

\begin{lemma}
If we use the prior $p_1$ in \eqref{eqn:us_prior}, then for every $n$,
\begin{equation*}
-Z_n \leq \|w^*\|_0(d+1)\log(4\sqrt{n}) + 4\eta + 2\lambda\sqrt{n}\,.
\end{equation*}
\label{lem:z_us}
\end{lemma}

\begin{proof}
For each $c \in (0, 1]$, we define the event
\begin{equation*}
\mathcal{E}_c := \{\nu \in N_c\}\,.
\end{equation*}

When $w \sim p_1(w\mid m=s) = \mathcal{U}(\mathbb{B}_1^s(1))$ and $\theta_i \sim p_1(\theta) = \mathcal{U}(\mathbb{B}_2^d(1))$, we have
\begin{align*}
\mathbb{P}_{\nu \sim p_1\mid m=s}[\mathcal{E}_c] = \tfrac{\mathrm{Vol}(\mathcal{W}_c)}{\mathrm{Vol}(\mathbb{B}_1^s(1))}\times\tfrac{\mathrm{Vol}(\Theta_{1,c})}{\mathrm{Vol}(\mathbb{B}_2^d(1))} \times \cdots \times \tfrac{\mathrm{Vol}(\Theta_{s,c})}{\mathrm{Vol}(\mathbb{B}_2^d(1))} = c^{s(d+1)}\,.
\end{align*}

Using this, and the fact that $p_1(m) = 2^{-m}$, we have
\begin{align*}
Z_n &= \mathbb{E}\left[\log\mathbb{E}_{\nu \sim p_1}\left[\exp(-{\textstyle \sum_{t=1}^{n}}\Delta L(\nu, X_t, A_t, Y_t))\right]\right]\\
&\geq \mathbb{E}\left[\log p_1(s)\mathbb{E}_{\nu \sim p_1\mid m=s}\left[\exp(-{\textstyle \sum_{t=1}^{n}}\Delta L(\nu, X_t, A_t, Y_t))\right]\right]\\
&\geq \mathbb{E}\left[\log p_1(s)\mathbb{P}_{\nu \sim p_1\mid m=s}[\mathcal{E}_c]\min_{\nu \in N_c}\left\{\exp(-{\textstyle \sum_{t=1}^{n}}\Delta L(\nu, X_t, A_t, Y_t))\right\}\right]\\
&= \log(p_1(s)) + \log(\mathbb{P}_{\nu \sim p_1\mid m=s}[\mathcal{E}_c]) - \mathbb{E}\left[\max_{\nu \in N_c}\left\{{\textstyle \sum_{t=1}^{n}}\Delta L(\nu, X_t, A_t, Y_t)\right\}\right]\\
&= -s\log(2) - s(d+1)\log(1/c) - \mathbb{E}\left[\max_{\nu \in N_c}\left\{{\textstyle \sum_{t=1}^{n}}\Delta L(\nu, X_t, A_t, Y_t)\right\}\right]
\end{align*}

Next, using Lemma \ref{lem:dl_alternative} and then Lemma \ref{lem:lip}, for any $\nu \in N_c$, we have
\begin{align*}
{\textstyle \sum_{t=1}^{n}}\Delta L(\nu, X_t, A_t, Y_t) &= {\textstyle \sum_{t=1}^{n}}\eta (f_{\nu}(X_t, A_t)-f_{\nu^*}(X_t, A_t))^2 + \lambda(f_\nu(X_t) - f_{\nu^*}(X_t))\\
&+ {\textstyle \sum_{t=1}^{n}}-2\eta(f_{\nu}(X_t, A_t) - f^*(X_t, A_t))\epsilon_t\\
&\leq 16c^2\eta n + 4c\lambda n + {\textstyle \sum_{t=1}^{n}}-2\eta(f_{\nu}(X_t, A_t) - f^*(X_t, A_t))\epsilon_t\,.
\end{align*}

Therefore,
\begin{align*}
\mathbb{E}\left[\max_{\nu \in N_c}\left\{{\textstyle \sum_{t=1}^{n}}\Delta L(\nu, X_t, A_t, Y_t)\right\}\right] &\leq 16c^2\eta n + 4c\lambda n\\
&+ \mathbb{E}\left[\max_{\nu \in N_c}\left\{{\textstyle \sum_{t=1}^{n}}-2\eta(f_{\nu}(X_t, A_t) - f^*(X_t, A_t))\epsilon_t\right\}\right]\,.
\end{align*}

Using Lemma \ref{lem:us_gaussian_comp}, for any $c \in (0, 1]$ and $\Delta > 0$, we also have
\begin{equation*}
\mathbb{E}\left[\max_{\nu \in N_c}\left\{{\textstyle \sum_{t=1}^{n}}-2\eta(f_{\nu}(X_t, A_t) - f^*(X_t, A_t))\epsilon_t\right\}\right] \leq s(d+1)\log(1 + \tfrac{2c}{\Delta}) + 8c^2\eta^2n + 2\eta\Delta n\,.
\end{equation*}

If we choose $c = 1/(2\sqrt{n})$ and $\Delta = 1/n$, then this bound becomes
\begin{equation*}
\mathbb{E}\left[\max_{\nu \in N_c}\left\{{\textstyle \sum_{t=1}^{n}}-2\eta(f_{\nu}(X_t, A_t) - f^*(X_t, A_t))\epsilon_t\right\}\right] \leq s(d+1)\log(1 + \sqrt{n}) + 2\eta^2 + 2\eta\,.
\end{equation*}

Thus with the choice $c = 1/(2\sqrt{n})$, we obtain the bound
\begin{equation*}
\mathbb{E}\left[\max_{\nu \in N_c}\left\{{\textstyle \sum_{t=1}^{n}}-2\eta(f_{\nu}(X_t, A_t) - f^*(X_t, A_t))\epsilon_t\right\}\right] \leq s(d+1)\log(1 + \sqrt{n}) + 2\eta^2 + 6\eta + 2\lambda\sqrt{n}\,.
\end{equation*}

If we combine everything, then we obtain
\begin{align*}
-Z_n &\leq s\log(2) + s(d+1)\log(2\sqrt{n}) + s(d+1)\log(1 + \sqrt{n}) + 2\eta^2 + 6\eta + 2\lambda\sqrt{n}\\
&\leq s(d+1)(\log(2) + 2\log(2\sqrt{n})) + 2\eta^2 + 6\eta + 2\lambda\sqrt{n}\\
&= s(d+1)\log(8n) + 2\eta^2 + 6\eta + 2\lambda\sqrt{n}\,.
\end{align*}
\end{proof}

\subsection{Proof of Theorem \ref{thm:upperbounds_us}}
\label{sec:us_ub_proof_sub}

\begin{proof}
Using Theorem \ref{thm:expregretbound}, and then Lemma \ref{lem:z_us} (with $\eta = 1/4$), we have
\begin{align*}
R_n(f^*) &\leq \lambda(4K + 6)n - \frac{1}{\lambda}Z_n\\
&\leq \lambda(4K + 6)n + \frac{1}{\lambda}\big(\|w^*\|_0(d+1)\log(8n) + 2\big) + 2\sqrt{n}\,.
\end{align*}

If we choose
\begin{equation*}
\lambda = \sqrt{\frac{(d+1)\log(8n)}{(4K + 6)n}}\,,
\end{equation*}

then we obtain the regret bound
\begin{align*}
R_n(f^*) &\leq (\|w^*\|_0 + 1)\sqrt{(4K + 6)(d+1)n\log(8n)} + 2\sqrt{\frac{(4K + 6)n}{\log(8n)}} + 2\sqrt{n}\\
&= \mathcal{O}\Big(\|w^*\|_0\sqrt{Kdn\log(n)}\Big)\,.
\end{align*}

If $s$ is a known upper bound on $\|w^*\|_0$ and we choose
\begin{equation*}
\lambda = \sqrt{\frac{s(d+1)\log(8n) + 2}{(4K + 6)n}}\,,
\end{equation*}

then we obtain the regret bound
\begin{align*}
R_n(f^*) &\leq 2\sqrt{(4K + 6)(s(d+1)\log(8n) + 2)n} + 2\sqrt{n}\\
&= \mathcal{O}(\sqrt{Ksdn\log(n)})\,.
\end{align*}
\end{proof}



\end{document}